%% file: ICML.tex
\documentclass{article}

\usepackage{microtype}
\usepackage{graphicx}
\usepackage{booktabs} 

\usepackage{hyperref}

\usepackage[accepted]{icml2021}

\icmltitlerunning{Enhancing Robustness of Neural Networks through Fourier Stabilization}

\usepackage{amsmath,amssymb,amsthm,mathtools}
\usepackage{enumitem}
\setlist{nosep}
\usepackage[table,cxdraw]{xcolor} 
\usepackage{caption}
\usepackage{subcaption}

\newif\ifarXiv
\arXivtrue

\usepackage{pgfplots}
\pgfplotsset{compat=1.5.1}

\usepackage{tcolorbox} 

\newcommand{\bE}{\mathbb{E}}

\newcommand{\bR}{\mathbb{R}}

\newcommand{\cH}{\mathcal{H}}

\newcommand{\cS}{\mathcal{S}}

\newcommand{\bolda}{\mathbf{a}}

\newcommand{\boldc}{\mathbf{c}}

\newcommand{\bolde}{\mathbf{e}}
\newcommand{\boldf}{\mathbf{f}}

\newcommand{\boldh}{\mathbf{h}}

\newcommand{\boldu}{\mathbf{u}}
\newcommand{\boldv}{\mathbf{v}}
\newcommand{\boldw}{\mathbf{w}}
\newcommand{\boldx}{\mathbf{x}}

\newcommand{\boldz}{\mathbf{z}}

\newcommand{\boldC}{\mathbf{C}}
\newcommand{\boldD}{\mathbf{D}}

\newcommand{\boldU}{\mathbf{U}}

\newcommand{\onenorm}[1]{\lVert #1 \rVert_1}
\newcommand{\twonorm}[1]{\lVert #1 \rVert_2}
\newcommand{\infnorm}[1]{\lVert #1 \rVert_\infty}
\newcommand{\pnorm}[1]{\lVert #1 \rVert_p}
\newcommand{\qnorm}[1]{\lVert #1 \rVert_q}

\DeclareMathOperator{\infl}{\textbf{Inf}}
\DeclareMathOperator{\sign}{sign}
\DeclareMathOperator{\argmax}{argmax}
\DeclareMathOperator{\sigmoid}{sigmoid}

\DeclareMathOperator{\var}{Var}
\DeclarePairedDelimiter{\norm}{\lVert}{\rVert} 

\allowdisplaybreaks

\newtheorem{theorem}{Theorem}

\newtheorem{lemma}{Lemma}

\newtheorem{corollary}{Corollary}

\begin{document}

\twocolumn[
\icmltitle{Enhancing Robustness of Neural Networks through Fourier Stabilization}




\begin{icmlauthorlist}
\icmlauthor{Netanel Raviv}{W}
\icmlauthor{Aidan Kelley}{W}
\icmlauthor{Michael Guo}{W}
\icmlauthor{Yevgeny Vorobeychik}{W}
\end{icmlauthorlist}

\icmlaffiliation{W}{Department of Computer Science and Engineering, Washington University in St. Louis, 1 Brookings Dr., St. Louis, MO 63103}
\icmlcorrespondingauthor{Netanel Raviv}{netanel.raviv@wustl.edu}

\icmlkeywords{Machine Learning, ICML}

\vskip 0.3in
]



\printAffiliationsAndNotice{}  

\begin{abstract}
   Despite the considerable success of  neural networks in security settings such as malware detection, such models have proved vulnerable to evasion attacks, in which attackers make slight changes to inputs (e.g., malware) to bypass detection.
   We propose a novel approach, \emph{Fourier stabilization}, for designing evasion-robust neural networks with binary inputs.  This approach, which is complementary to other forms of defense, replaces the weights of individual neurons with robust analogs derived using Fourier analytic tools.
   The choice of which neurons to stabilize in a neural network is then a combinatorial optimization problem, and we propose several methods for approximately solving it.
   We provide a formal bound on the per-neuron drop in accuracy due to Fourier stabilization, and experimentally demonstrate the effectiveness of the proposed approach in boosting robustness of neural networks in several detection settings.
   Moreover, we show that our approach effectively composes with adversarial training.

\end{abstract}

\section{Introduction}
Deep neural network models demonstrate human-transcending capabilities in many applications, but are often vulnerable to attacks that involve small (in $\ell_p$-norm) adversarial perturbations to inputs~\cite{Intriguing,Goodfellow15,Madry}. 
This issue is particularly acute in security applications, where a common task is to determine whether a particular input (e.g., executable, twitter post) is malicious or benign.
In these settings, malicious parties have a strong incentive to redesign inputs (such as malware) in order to \emph{evade} detection by deep neural network-based detectors, and there have now been a series of demonstrations of successful evasion attacks~\cite{grosse2016adversarial,li2018evasion,laskov2014practical,xu2016automatically}.
In response, a number of approaches have been proposed to create models that are more robust to evasion attacks~\cite{Cohen2019CertifiedAR,Lecuyer19,Raghunathan18,Wong18,Wong18b}, with methods using adversarial training---where models are trained by replacing regular training inputs with their adversarially perturbed variants---remaining the state of the art~\cite{Goodfellow15,Madry,tong2019improving,Eugene}.
Nevertheless, despite considerable advances, increasing robustness of deep neural networks to evasion attacks typically entails a considerable decrease in accuracy on unperturbed (clean) inputs~\cite{Madry,Wu20}.


We propose a novel approach for enhancing robustness of deep neural networks with binary inputs to adversarial evasion that leverages Fourier analysis of Boolean functions~\cite{AnalysisOfBoolean}.
Unlike most prior approaches for boosting robustness, which aim to refactor the entire deep neural network, say, through adversarial training, our approach is more fine-grained, applied at the level of individual neurons.
Specifically, we start by treating neurons as linear classifiers over binary inputs, and considering their robustness as the problem of maximizing the average distance of \emph{all inputs in the input space} from the neuron's decision boundary.
We then derive a closed-form solution to this optimization problem; the process of replacing the original weights by their more robust variants, given by this solution, is called \emph{Fourier stabilization of neurons}. Further, a bound for the per-neuron drop in accuracy due to this process is derived.

This idea applies to most common activation functions, such as~$\operatorname{logistic}, \tanh$, $\operatorname{erf}$, and $\operatorname{ReLU}$ (treating activation as a binary decision).
Finally, we determine which subset of neurons in a neural network to stabilize.
While this is a difficult combinatorial optimization problem, we develop several effective algorithmic approaches for it.

Our full approach, which we call \emph{Fourier stabilization of a neural network} (abbrv.~stabilization), applies only to neural networks with binary inputs, and is targeted at security applications, where binary inputs are common and, indeed, it is often the case that binarized inputs outperform real-valued alternatives~\cite{Tong38,tong2019improving}.
We emphasize that our approach is \emph{complementary} to alternative defenses: it applies \emph{post-training}, and can thus be easily composed with any defense, such as adversarial example detection~\cite{Xu18} or adversarial training.
Moreover, as our approach does not require any training data (as it stabilizes neurons directly), it can even apply to settings where one has a neural network that needs to be made more robust, but not training data, which is sensitive (e.g., in medical and cybersecurity applications where data contains sensitive or classified information). Access to training data, however, enables the additional benefit of estimating robustness and accuracy in practice; we use this approach in our experiments to decide which subset of neurons to stabilize.

We experimentally evaluate the proposed \emph{Fourier stabilization} approach on several datasets involving detection of malicious inputs, including malware detection and hate speech detection.
Our experiments show that our approach considerably improves neural network robustness to evasion in these domains, and effectively composes with adversarial training defense.

\paragraph{Our contribution}

We begin in Section~\ref{section:preliminaries} by familiarizing the reader with the formal definition of robustness (specifically, \textit{prediction change} by \cite{AdversarialRiskBoolean}), its geometric interpretation, and provide some necessary background on Fourier analysis of Boolean functions. We proceed in Section~\ref{section:technical} by formulating the stabilization of neurons as an optimization problem, and solving it analytically for the~$\ell_1$-metric in Section~\ref{section:stabilization} (the solution for all other~$\ell_p$-metrics is given in\ifarXiv the appendix\else~\cite{arXivVersion}\fi). In Section~\ref{section:accuracy} we employ probabilistic tools from~\cite{AnalysisOfBoolean,ChowParameters,TestingHalfspaces} (among others) to bound the loss of accuracy that results from stabilization of a neuron, i.e., the fraction of inputs that would lie on the ``wrong'' side of its original decision boundary. 

In Section~\ref{section:NN} the discussion is extended to neural networks. It is observed that stabilizing the entire first layer might not be effective for improving robustness while maintaining accuracy. Instead, one should find an optimal subset of those, whose stabilization increases robustness the most, while maintaining bounded loss of accuracy. Since this combinatorial optimization problem is hard to solve in general, we suggest a few heuristics. The efficacy of these heuristics is demonstrated in Section~\ref{section:experiments} by showing improved accuracy-robustness tradeoff in classifying several commonly used cybersecurity datasets under state-of-the-art attacks. Further, it is also demonstrated that these techniques can be effectively used in conjunction with adversarial training. Future research directions are discussed in Section~\ref{section:discussion}.


%

\section{Preliminaries}\label{section:preliminaries}
For~$\boldw\in\bR^n$ and~$\theta\in\bR$, denote the hyperplane~$\cH=\{ \boldx\in\bR^n\vert \boldx\boldw^\intercal=\theta \}$ by~$\cH(\boldw,\theta)$. 
Our fundamental technique operates at the level of neurons in a neural network, which we treat as (generalized) linear models.
We start by considering  linear models of the form $h(\boldx)=\sign(\boldx\boldw^\intercal-\theta)$ that map binary inputs $\boldx\in\{ \pm 1 \}^n$ to binary outputs; below, we discuss how the machinery we develop applies to a variety of activation functions.
For~$1\le p\le  \infty$ let~$d_p$ and~$\pnorm{\cdot}$ be the~$\ell_p$-distance and~$\ell_p$-norm, respectively. That is, for vectors~$\boldv=(v_i)_{i=1}^n$ and~$\boldu=(u_i)_{i=1}^n$ let~$\pnorm{\boldv}=(\sum_{i=1}^n|v_i|^p)^{1/p}$ (or $\max\{ |v_i| \}_{i=1}^n$ if~$p=\infty$) and~$d_p(\boldv,\boldu)=\pnorm{\boldv-\boldu}$. 
For real numbers~$q,p\ge 1$, the norms~$\ell_p$ and~$\ell_q$ are called \textit{dual} if~$\frac{1}{p}+\frac{1}{q}=1$.
For example, the dual norm of~$\ell_2$ is itself, and the dual norm of~$\ell_1$ is~$\ell_\infty$. In the remainder of this paper,~$\ell_p$ and~$\ell_q$ denote dual norms.
We will make use of the following theorem:
\begin{theorem}\label{theorem:lp}
	\cite{LpDistance} (Sec.~5) For a hyperplane~$\cH(\boldv,\mu)\subseteq\bR^n$, a point~$\boldz\in\bR^n$, and any~$p\ge 1$, let $d_p(\boldz,\cH(\boldv,\mu))$ denote the~$\ell_p$-distance of~$\cH(\boldv,\mu)$ from~$\boldz$, i.e., $\min\{ d_p(\boldu,\boldz)\vert\boldu\in\cH(\boldv,\mu) \}$. Then, we have $d_p(\boldz,\cH(\boldv,\mu))=\frac{|\boldz\cdot\boldv^\intercal-\mu|}{\norm{\boldv}_q}$.
\end{theorem}

\subsection{Definition of Robustness}\label{section:defOfRobustness}
We operate under the geometric interpretation of robustness, in which the adversary is given a random~$\boldx\in\{ \pm1  \}^n$, and would like to apply minimum~$\ell_p$-change to induce misclassification. Since we address binary inputs, we focus our attention on~$p=1$, even though our techniques are also applicable to~$1<p\le \infty$. The case~$p=1$ simultaneously captures bit flips, where the adversary changes a the sign of an entry, and bit erasures, where the adversary changes an entry to zero. Notice that a bit flip causes~$\ell_1$-perturbation of~$2$, and a bit erasure causes~$\ell_1$-perturbation of~$1$.

We use one of the standard definitions of robustness of a classifier $h$ at an input $\boldx$ as the smallest distance of $\boldx$ to the decision boundary~\cite{AdversarialRiskBoolean}.
Formally, the \textit{prediction change robustness} (henceforth, simply \emph{robustness}) of a model~$h$ is defined as
\begin{align}\label{equation:robustnessDef}
	\bE_{\boldx}\inf\left\{ r:\exists \boldx'\in \operatorname{Ball}^p_r(\boldx), h(\boldx')\ne h(\boldx) \right\},
\end{align}
where~$\operatorname{Ball}^p_r(\boldx)$ is the set of all elements of~$\bR^n$ that are of~$\ell_p$-distance at most~$r$ from~$\boldx$. 
Note that in our setting, \eqref{equation:robustnessDef} is equivalent to the $\ell_p$-distance from the decision boundary (hyperplane), i.e.,~$\bE_\boldx d_p(\boldx,\cH(\boldw,\theta))$.
A natural goal for robustness is therefore to \emph{maximize} the expected $\ell_p$-distance to the decision boundary.
This problem will be the focus of \emph{Fourier stabilization of neurons} below.


\subsection{Fourier analysis of Boolean functions.}\label{section:FourierBackground}

Since subsequent sections rely on notions from Fourier analysis of Boolean functions, we provide a brief introduction. For a thorough treatment of the topic the reader if referred to~\cite{AnalysisOfBoolean}. 
Let $[n]$ denote the set $\{1,\ldots,n\}$.
Every function~$f\colon\{ \pm 1\}^n\to \bR$ can be represented as a linear combination over~$\bR$ of the functions~$\{ \chi_{\cS}(\boldx) \}_{\cS\subseteq [n]}$, where~$\chi_\cS(\boldx)=\prod_{j\in \cS}x_j$ for every~$\cS\subseteq [n]$. The coefficient of~$\chi_\cS(\boldx)$ in this linear combination is called the \textit{Fourier coefficient} of~$f$ at~$\cS$, and it is denoted by~$\Hat{f}(\cS)$. 
Each Fourier coefficient~$\Hat{f}(\cS)$ equals the inner product between~$f$ and~$\chi_\cS$, defined as~$\bE_{\boldx}f(\boldx)\chi_\cS(\boldx)$, where~$\boldx$ is chosen \textit{uniformly} at random. The inner product between functions~$f$ and~$g$ equals the inner product (in the usual sense) between their respective Fourier coefficients, a result known as Plancherel's identity: $\bE_\boldx{f(\boldx)g(\boldx)}=\sum_{\cS\subseteq [n]}\Hat{f}(\cS)\Hat{g}(\cS)$. 

For brevity, we denote~$\hat{f}(\{ i\} )=\hat{f}_i$ for every~$i\in[n]$ and~$\hat{f}_\varnothing=\hat{f}(\varnothing)$. We also define the vector~$\hat{\boldf}\triangleq (\hat{f}_1,\ldots,\hat{f}_n)$. The entries of~$\hat{\boldf}$, known as \textit{Chow parameters}, play an important role in the analysis of Boolean functions in general, and of~$\sign$ functions in particular (e.g.,~\cite{ChowParameters}). We also note that when the range of~$f$ is small (e.g. $f:\{ \pm 1 \}^n\to[-1,1]$, as in $\sigmoid$ functions), Hoeffding's inequality implies that any Fourier coefficient~$\hat{f}(\cS)$ can be efficiently approximated by choosing many~$\boldx$'s uniformly at random from~$\{\pm1 \}^n$, and averaging the expressions~$f(\boldx)\chi_\cS(\boldx)$. Finally, in the sequel we make use of the following lemma, whose proof is given in \ifarXiv Appendix~\ref{appendix:omittedProofs}\else \cite{arXivVersion}\fi.

\begin{lemma}\label{lemma:LTFfacts}
	For~$h(\boldx)=\sign(\boldx\boldw^\intercal-\theta)$ we have that $\sign(\hat{h}_i)=\sign(w_i)$ for every~$i\in[n]$.
\end{lemma}

\section{Increasing Robustness of Individual Neurons}\label{section:technical}

Recall that our goal is to increase robustness, quantified as the expected distance from the decision boundary, of individual neurons.
Suppose for now that a neuron is a linear classifier $h(\boldx)=\sign(\boldx\boldw^\intercal-\theta)$.
Then, by Theorem~\ref{theorem:lp}, the distance from the decision boundary for a given input $\boldx$ is
\begin{align}\label{equation:signedDistance}
	d_p(\boldx,\cH(\boldw,\theta))=\frac{|\boldx\boldw^\intercal-\theta|}{\qnorm{\boldw}}=\frac{\boldx\boldw^\intercal-\theta}{\qnorm{\boldw}}\cdot h(\boldx).
\end{align}
In actuality, we wish to measure this distance with respect to \emph{all} inputs in the input space.
We can formalize this as the \emph{average} distance over the input space (which is finite, since inputs are binary), or, equivalently if $\qnorm{\boldw} = 1$, as  $\bE_{\boldx}(\boldx\boldw^\intercal-\theta)\cdot h(\boldx)$, where the expectation is with respect to the uniform distribution over inputs.\footnote{One may be concerned about the use of a uniform distribution over inputs. However, our experimental evaluation below demonstrates effectiveness for several real datasets. Additionally, we note that in some cases, a simple uniformization mechanism can be applied (see\ifarXiv~Appendix~\ref{appendix:uniform}\else~\cite{arXivVersion}\fi) as part of feature extraction.
}

Now, suppose that we are given a neuron parametrized by $(\boldw,\theta)$ as input, and we wish to transform it in order to maximize its robustness---that is, average distance to the hyperplane---by choosing new weights and bias, $(\boldv,\mu)$.
We can formalize this as the following optimization problem:
\begin{tcolorbox}[colback=blue!5!white,colframe=blue!50!white,title=\textbf{The Neuron-Optimization Problem}]
	\textbf{Input:} A neuron~$h(\boldx)=\sign(\boldx\boldw^\intercal-\theta)$.
	
	\textbf{Variables:} $\boldv=(v_1,\ldots,v_n)\in\bR^n$.
	
	\textbf{Objective:} Maximize $ \bE_{\boldx} (\boldx\boldv^\intercal-\mu)h(\boldx)$.
	
	\textbf{Constraints:} 
	\begin{itemize}
		\item[--] If~$p>1$ (including~$p=\infty$): $\qnorm{\boldv}^q=1$.
		\item[--] If~$p=1$: $\infnorm{\boldv}=1$.
	\end{itemize}
\end{tcolorbox}

However, an issue arises in finding the optimal bias~$\mu^*$: treating~$\mu$ as an unbounded variable will result in an expression that can be made arbitrarily large by taking~$\mu$ to either~$\infty$ (if~$\sum_{\boldx}h(\boldx)>0$) or~$-\infty$ (otherwise). 
Therefore, in what follows we treat~$\mu$ as a constant, and discuss its optimal value with respect to the loss of accuracy in Section~\ref{section:accuracy}. 

We briefly note here a connection to support vector machines (SVMs), which are based on an analogous margin maximization idea.
The key distinction is that we aim to maximize margin with respect to the \emph{entire} input space \emph{given a fixed trained model}, whereas SVM maximizes margin with respect to a given dataset in order to train a model.
Thus, our approach is about robust generalization rather than training.


\subsection{Fourier Stabilization of Neurons}\label{section:stabilization}
We now derive an analytic solution to the optimization problem above using Fourier analytic techniques.
Since we use a uniform distribution over~$\boldx\in\{ \pm1 \}^n$, our objective function becomes
\begin{align*}
	\bE_\boldx(\boldx\boldv^\intercal-\mu)h(\boldx)
	=\hat{\boldh}\boldv^\intercal-\hat{h}_\varnothing\mu,
\end{align*}
by a straightforward application of Plancherel's identity. Therefore, the optimization problem reduces to linear maximization under equality constraints. In what follows, this maximization problem is solved analytically; we emphasize once more that~$p=1$ is the focus of our attention, and yet the solution is stated in greater generality for completeness. Fourier stabilization for~$p\ne1$ is potentially useful in niche applications such as neural computation in hardware and adversarial noise in weights. We provide the proof for the case~$p=1$, and the remaining cases ($1<p\le\infty$) are discussed in\ifarXiv~Appendix~\ref{appendix:omittedProofs}\else~\cite{arXivVersion}\fi.

\begin{theorem} \label{theorem:main}
	Let~$h(\boldx)=\sign(\boldx\boldw^\intercal-\theta)$, and~$\hat{\boldh}=(\hat{h}_1,\ldots,\hat{h}_n)$. The solution~$\boldw^*=(w_1^*,\ldots,w^*_n)$ to the neuron-optimization problem is
	\begin{align*}
		w_i^*=
		\begin{cases}
			\sign(\hat{h}_i)\cdot\left(\frac{|\hat{h}_i|}{ \norm{\hat{\boldh} }_p }\right)^{p-1} & \mbox{if }1\le p<\infty\\
			0 & \mbox{if }p=\infty \mbox{ and }|\hat{h}_i|< \infnorm{\hat{\boldh}}\\
			|\hat{h}_i| & \mbox{if }p=\infty \mbox{ and }|\hat{h}_i|= \infnorm{\hat{\boldh}}
		\end{cases}
	\end{align*}
	Further, the maximum value of the objective is~$\pnorm{\hat{\boldh}}-\hat{h}_\varnothing\mu$.
\end{theorem}

\begin{proof}[Proof for~$p=1$]
	Notice that the constraint $\infnorm{\boldv}=1$ translates to the~$n$ constraints~$-1\le v_i\le 1$, where at least one of which must be attained with equality; this is guaranteed since the optimum of a linear function over a convex polytope is always obtained on the boundary. Hence, the optimization problem reduces to a linear objective function under box constraints. Therefore, to maximize~$\hat{\boldh}\boldv^\intercal-\hat{h}_\varnothing\mu$, it is readily verified that every~$v_i$ must be~$\sign(\hat{h}_i)$. The solution in this case is~$\boldw^*=(\sign(\hat{h}_i))_{i=1}^n$, and the resulting objective is~$\hat{\boldh}\boldv^\intercal-\hat{h}_\varnothing\mu=\onenorm{\hat{\boldh}}-\hat{h}_\varnothing\mu$. 
\end{proof}
We refer to the solution in Theorem~\ref{theorem:main} as \emph{Fourier stabilization of neurons} (or simply \emph{stabilization}), and the associated neuron as \emph{stabilized}.
If we fix~$\mu=\theta$ it is easily proved (see\ifarXiv~Appendix~\ref{appendix:omittedProofs}\else~\cite{arXivVersion}\fi) that stabilization increases robustness.
\begin{lemma}\label{lemma:stabWorks}
	For every~$h(\boldx)=\sign(\boldx\boldw^\intercal-\theta)$, its stabilized counterpart~$h'(\boldx)=\sign(\boldx\boldw^{*\intercal}-\theta)$ is at least as robust as~$h(\boldx)$. In particular:
	\begin{align*}
		\bE_\boldx d_p(\boldx,\cH(\boldw,\theta))\le \pnorm{\hat{\boldh}}-\hat{h}_\varnothing\theta\le \bE_\boldx d_p(\boldx,\cH(\boldw^*,\theta)).
	\end{align*}
\end{lemma}

Notice that thanks to Lemma~\ref{lemma:LTFfacts}, for~$p=1$ it is not necessary to approximate the Fourier coefficients of~$h$ since their sign is given by the sign of the respective entries of~$\boldw$. 
Notice also that in this case the resulting model is \textit{binarized}, i.e., all its weights are~$\{\pm 1\}$. Such models are popular as neurons in \textit{binarized neural networks}~\cite{BNN}, and our results shed some light on their apparent increased robustness~\cite{AttackingBNN}. 

Also notice that while our formal analysis pertains to~$\sign(\cdot)$, similar reasoning can be applied as a heuristic to many other activation functions, and in particular to sigmoid functions (such as~$\operatorname{logistic}(\cdot),\tanh(\cdot)$, etc.). For example, one can replace $\frac{1}{1+e^{-(\boldx\boldw^\intercal-\theta)}}$ by $\frac{1}{1+e^{-(\boldx\boldw^{*\intercal}-\theta)}}$, where~$\boldw^*$ is the solution of the neuron-optimization problem when applied over~$\sign(\boldx\boldw^\intercal-\theta)$. Since the outputs of sigmoid functions are very close to~$\pm1$ for most inputs, adversarial attacks  attempt to push these inputs towards~$\cH(\boldw,\theta)$, a task which is made harder by stabilization. Furthermore, \textit{one-sided} robustness is increased by stabilizing~$\operatorname{ReLU}(\boldx)=\max\{ 0,\boldx\boldw^\intercal-\theta \}$; $\boldx$'s for which~$\boldx\boldw^\intercal<\theta$ must be shifted across~$\cH(\boldw,\theta)$ for the output of the neuron to change. Hence, stabilizing~$\operatorname{ReLU(\cdot)}$, i.e., replacing $\max\{ 0,\boldx\boldw^\intercal-\theta \}$ by $\max\{ 0,\boldx\boldw^{*\intercal}-\theta \}$, increases the robustness of attacking such inputs.

\subsection{Bounding the Loss in Accuracy}\label{section:accuracy}

In the above discussion we optimized for robustness, but were oblivious to the loss of accuracy, and did not specify the bias~$\mu$. In this section we again focus on~$p=1$, and the remaining cases are given in\ifarXiv~Appendix~\ref{appendix:Loss}\else~\cite{arXivVersion}\fi. We now quantify the accuracy loss of a \textit{single neuron}.
Accuracy-loss of a neuron~$h(\boldx)$ is quantified in the following sense: we bound the fraction of~$\boldx$'s such that~$h(\boldx)\ne h'(\boldx)$, i.e., they are on the wrong side of the original decision boundary~$\cH(\boldw,\theta)$ due to the stabilization. The bound is given as a function of the Fourier coefficients of~$h$, and of the bias~$\mu$ that can be chosen freely. The choice of~$\mu$ manifests a robustness-accuracy tradeoff which we discuss subsequently (Corollary~\ref{corollary:RobustnessAccuracy}). Proving the bound requires the following technical lemmas. 

\begin{lemma}\label{lemma:C0}
	Let~$\ell(\boldx)=\sum_{i=1}^n a_ix_i$, with~$\sum_{i=1}^n a_i^2=1$ and~$|a_i|\le\epsilon$. If the entries of~$\boldx$ are chosen uniformly at random, then there exist a constant~$C_0\approx0.47$ such that for every~$\mu\ge0$,
	\begin{align*}
		\Pr[|\ell(\boldx)-\mu|\le u]\le u\sqrt{\frac{2}{\pi}}+2C_0\epsilon\mbox{ for every }~u>0.
	\end{align*}
\end{lemma}
\begin{proof}
	Notice that
	\begin{align*}
		\Pr[|\ell(\boldx)&-\mu|\le u]=\Pr[\mu-u\le\ell(\boldx)\le\mu+u]\\
		&\overset{(a)}{\le}\Pr[\mu-u\le N(0,1) \le \mu+u]+2C_0\epsilon\\
		&=\int_{\mu-u}^{\mu+u}\tfrac{1}{\sqrt{2\pi}}e^{-x^2/2}dx+2C_0\epsilon\\
		&\overset{(b)}{\le}u\sqrt{\frac{2}{\pi}}+2C_0\epsilon,
	\end{align*}
	where~$(a)$ follows from The Berry-Esseen Theorem\footnote{A parametric variant of the central limit theorem\ifarXiv; it is cited in full in Appendix~\ref{appendix:Loss}, Theorem~\ref{theorem:BEthm2}\fi.}, and~$(b)$ follows since~$e^{-x^2/2}\le 1$.
\end{proof}

\begin{lemma}\label{lemma:auxClaims}
	Let~$Z_1,\ldots,Z_n$ be independent and uniform~$\{\pm \frac{1}{\sqrt{n}}\}$ random variables, let~$\boldz=(Z_1,\ldots,Z_n)$, and let~$S=\sum_{i=1}^n Z_i$.
	\begin{enumerate}
		\item[A.] For every~$\bolda\in\{ \pm1 \}$ the random variables~$S$ and~$\bolda\boldz^\intercal$ are identically distributed.
		\item[B.] For every~$\mu$ we have~$\bE[|S-\mu|]=\alpha(\mu)$, where
		\begin{align*}
			\alpha(\mu)=\frac{1}{2^n}\cdot\sum_{i\in\{-n,-n+2,\ldots,n\}} \binom{n}{\frac{n-i}{2}}\cdot|i\sqrt{n}-\mu|
		\end{align*}
	\end{enumerate}
\end{lemma}

\begin{proof}~
	\begin{enumerate}
		\item[A.] Since each~$Z_i$ is uniform over~$\{ \pm \frac{1}{\sqrt{n}} \}$, it follows that the random variables~$Z_i$ and~$-Z_i$ are identically distributed for every~$i$, which implies the claim since the~$Z_i$'s are independent.
		\item[B.] Follows by a straightforward computation of the expectation.\qedhere
	\end{enumerate}
\end{proof}




We mention that the proof of the following theorem is strongly inspired by a well-known~$p=2$ counterpart, that appeares repeatedly in the theoretical computer science literature (e.g., \cite{TestingHalfspaces} (Thm.~26, Thm.~34, Thm.~49), \cite{ChowParameters} (Thm.~8.1), and~\cite{AnalysisOfBoolean} ($\frac{2}{\pi}$-Thm.), among others).

\begin{theorem}\label{theorem:onenorm}
	For~$h(\boldx)=\sign(\boldx\boldw^\intercal-\theta)$ let~$\ell(\boldx)=\frac{1}{\sqrt{n}}\cdot\boldx\boldw^{*\intercal}$, where~$\boldw^*$ is given by Theorem~\ref{theorem:main}, and for any~$\mu$ let
	\begin{align}\label{equation:gammaDefinition}
		\gamma=\gamma(\mu)=\left|\tfrac{1}{\sqrt{n}}\onenorm{\hat{\boldh}}-\hat{h}_\varnothing\mu-\alpha(\mu)\right|,
	\end{align}
	where~$\alpha(\mu)$ is defined in Lemma~\ref{lemma:auxClaims}. Then, 
	\begin{align*}
		\Pr(\sign(\ell(\boldx)-\mu)\ne h(\boldx))\le\tfrac{3}{2}\left(\tfrac{C_0}{\sqrt{n}}+\sqrt{\tfrac{C_0^2}{n}+\sqrt{\tfrac{2}{\pi}}\cdot\gamma}\right).
	\end{align*}
\end{theorem}
\begin{proof}
	According to Plancherel's identity, we have that
	\begin{align}\label{equation:tauOneNorm}
		\bE&[h(\boldx)(\ell(\boldx)-\mu)]=\sum_{\cS\subseteq[n],\cS\ne \varnothing}\hat{h}(\cS)\hat{\ell}(\cS)-\hat{h}_\varnothing\mu\nonumber\\
		&=\tfrac{1}{\sqrt{n}}\sum_{i=1}^n \hat{h}_i\sign(\hat{h}_i)-\hat{h}_\varnothing\mu
		=\tfrac{1}{\sqrt{n}}\onenorm{\hat{\boldh}}-\hat{h}_\varnothing\mu.
	\end{align}
	Moreover, since Lemma~\ref{lemma:auxClaims} implies that
	\begin{align}\label{eqaution:alphanA}
		 \bE[|\ell(\boldx)-\mu|]=\alpha(\mu),
	\end{align}
	we have 
	\begin{align}\label{equation:gammaUpperBoundA}
		\bE&[(\ell(\boldx)-\mu)\cdot(\sign(\ell(\boldx)-\mu)-h(\boldx))]=\nonumber\\
		&=\bE[|\ell(\boldx)-\mu|]-\bE[h(\boldx)(\ell(\boldx)-\mu)]\nonumber\\
		&\overset{\eqref{equation:tauOneNorm},\eqref{eqaution:alphanA}}{=}\alpha(\mu)-\tfrac{1}{\sqrt{n}}\onenorm{\hat{\boldh}}+\hat{h}_\varnothing\mu\le \gamma.
	\end{align}
	In what follows, we bound~$\Pr(\sign(\ell(\boldx)-\mu)\ne h(\boldx))$ by studying the expectation in~\eqref{equation:gammaUpperBoundA}. According to Lemma~\ref{lemma:C0} with~$\epsilon=\tfrac{1}{\sqrt{n}}$, it follows that for every~$u>0$ (a precise~$u$ will be given shortly)
	\begin{align}\label{equation:probellA}
		\Pr(|\ell(\boldx)-\mu|\le u)
		<u\sqrt{\tfrac{2}{\pi}}+ \tfrac{2C_0}{\sqrt{n}}\triangleq\eta(u).
	\end{align}
	Assume for contradiction that $\Pr(\sign(\ell(\boldx)-\mu)\ne h(\boldx))> \frac{3}{2}\eta(u)$. Since~$\Pr(|\ell(\boldx)-\mu|>u)\ge 1-\eta(u)$ by~\eqref{equation:probellA}, it follows that
	\begin{align}\label{equation:ProbBoundA}
		\Pr(\sign(\ell(\boldx)-\mu)\ne h(\boldx)\mbox{ and } |\ell(\boldx)-\mu|>u) > \tfrac{\eta(u)}{2}.
	\end{align}
	Also, observe that
	\begin{align}\label{equation:EboundA}
		\bE&[(\ell(\boldx)-\mu)(\sign(\ell(\boldx)-\mu)-h(\boldx))]=\nonumber\\
		\tfrac{1}{2^n}&\left(\sum_{\boldx\vert \sign(\ell(\boldx)-\mu)>h(\boldx)} 2(\ell(\boldx)-\mu) -\right.\nonumber\\
		&\left.\sum_{\boldx\vert \sign(\ell(\boldx)-\mu)<h(\boldx)}2(\ell(\boldx)-\mu)\right).
	\end{align}
	Since all summands in left summation in~\eqref{equation:EboundA} are positive, and all summands in the right one are negative, by keeping in the left summation only summands for which~$\ell(\boldx)-\mu> u$, and in the right summation only those for which~$\ell(\boldx)-\mu<-u$, we get
	\begin{align}\label{equation:uetauA}
		\eqref{equation:EboundA}&\ge 2u\cdot\frac{\left|\left\{ \boldx~\Big\vert~
			\begin{matrix}
				\sign(\ell(\boldx)-\mu)\ne h(\boldx)\\
				\mbox{ and }|\ell(\boldx)-\mu|>u
			\end{matrix}
			 \right\}\right|}{2^n}\nonumber\\
		&\overset{\eqref{equation:ProbBoundA}}{>}u\cdot\eta(u).
	\end{align}
	Combining~\eqref{equation:uetauA} with~\eqref{equation:gammaUpperBoundA}, it follows that $u\cdot\eta(u) < \gamma$, which by the definition in~\eqref{equation:probellA} implies that
	\begin{align}\label{equation:possibleContA}
		\sqrt{\tfrac{2}{\pi}}\cdot u^2+\tfrac{2C_0}{\sqrt{n}}\cdot u-\gamma&< 0.
	\end{align}
	We wish to find the smallest positive value of~$u$ which contradicts~\eqref{equation:possibleContA}. By applying the textbook solution, we have that any positive~$u$ which complies with~\eqref{equation:possibleContA} must satisfy
	\begin{align}\label{equation:usValueA}
		u&<
		\frac{-\frac{C_0}{\sqrt{n}}+\sqrt{\frac{C_0^2}{n}+\sqrt{\frac{2}{\pi}}\cdot\gamma}}{\sqrt{\frac{2}{\pi}}}
	\end{align}
	and hence setting~$u$ to the right hand side of~\eqref{equation:usValueA} leads to a contradiction. Therefore,
	\begin{align*}
		\Pr(\sign(\ell(\boldx)&-\mu)\ne h(\boldx))\le \tfrac{3}{2}\eta(u)\overset{\eqref{equation:probellA}}{=}\tfrac{3}{2}(u\sqrt{\tfrac{2}{\pi}}+ \tfrac{2C_0}{\sqrt{n}})\\
		&=\tfrac{3}{2}\left(\tfrac{C_0}{\sqrt{n}}+\sqrt{\tfrac{C_0^2}{n}+\sqrt{\tfrac{2}{\pi}}\cdot\gamma}\right).\qedhere
	\end{align*}
\end{proof}

\begin{corollary}\label{corollary:RobustnessAccuracy}
	Theorem~\ref{theorem:onenorm} complements Theorem~\ref{theorem:main} in terms of the robustness-accuracy tradeoff in choosing the bias~$\mu$ of the stabilized neuron. Given $h(\boldx)=\sign(\boldx\boldw^\intercal-\theta)$, choosing~$\mu=\theta$ guarantees increased robustness of the stabilized model~$h'(\boldx)=\sign(\boldx\boldw^{*\intercal}-\mu)$ by Lemma~\ref{lemma:stabWorks}, and the accuracy loss is quantified by setting\footnote{More precisely, setting~$\mu=\theta/\sqrt{n}$, due to the additional normalization factor in Theorem~\ref{theorem:onenorm}.}~$\mu=\theta$ Theorem~\ref{theorem:onenorm}. However, one is free to choose any other~$\mu\ne \theta$, and obtain different accuracy and robustness. For any such~$\mu$, the robustness of the stabilized neuron is
	\begin{align*}
		\bE_\boldx d_p(\boldx,\cH(\boldw^*,\mu))=\sum_{i=1}^n w_i^*\hat{h}'_i-\hat{h}'_\varnothing\mu
	\end{align*}
	by Plancherel's identity, and the resulting accuracy loss is given similarly by Theorem~\ref{theorem:onenorm}. In any case, the resulting accuracy and robustness should be contrasted with those of the non-stabilized model, where the accuracy loss is obviously zero, and the robustness is
	\begin{align*}
		\bE_\boldx d_p(\boldx,\cH(\boldw,\theta))=\sum_{i=1}^n w_i\hat{h}_i-\hat{h}_\varnothing\theta.
	\end{align*}
%
\end{corollary}

\section{Fourier Stabilization of Deep Neural Networks}\label{section:NN}

Thus far, we were primarily focused on robustness and accuracy of individual neurons, modeled as linear classifiers.
We now consider the problem of increasing robustness of neural networks, comprised of a collection of such neurons.
The general idea is that by stabilizing individual neurons in the network we can increase the overall robustness.
However, increased robustness comes almost inevitably at some loss in accuracy, and different neurons in a network will face a somewhat different robustness-accuracy tradeoff.
Consequently, we will now consider the problem of stabilizing a neural network by selecting a subset of neurons to stabilize that best trades off robustness and accuracy.

To formalize this idea, let $\cS$ denote the subset of neurons that are chosen for stabilization.
Define $R(\cS)$ as robustness (for example, measured empirically on a dataset using any of the standard measures) and let $A(\cS)$ be the accuracy (again, measured empirically on unperturbed data) after we stabilize the neurons in set $\cS$.
Our goal is to maximize robustness subject to a constraint that accuracy is no lower than a predefined lower bound $\beta$:
\begin{tcolorbox}[colback=blue!5!white,colframe=blue!50!white,title=\textbf{The Network-Optimization Problem}]
	\textbf{Input:} A neural network~$\mathrm{N}$ with first-layer neurons~$\{h_i(\boldx)=\sign(\boldx\boldw_i^\intercal-\theta_i)\}_{i=1}^t$, and accuracy bound~$\beta$.\\
	\textbf{Variable:} $\cS\subseteq \{1,\ldots,t\}$.\\
	\textbf{Objective:} Maximize $R(\cS)$\\ 
	\textbf{Constraint:}
	$A(\cS) \ge \beta$.
\end{tcolorbox}
Observe that while in principle we can stabilize any subset of neurons, the tools we developed in Section~\ref{section:stabilization} apply only to neurons with binary inputs, which is, in general, only true of the neurons in the first (hidden) layer of the neural network.
Consequently, both the formulation above, and experiments below, focus on stabilizing a subset of the first-layer neurons.

There are two principal challenges in solving the optimization problem above.
First, it is a combinatorial optimization problem in which neither $R(\cS)$ nor $A(\cS)$ are guaranteed to have any particular structure (e.g., they are not even necessarily monotone).
Second, using empirical robustness $R(\cS)$ is typically impractical, as computing $\ell_1$ adversarial perturbations on binary inputs is itself a difficult combinatorial optimization problem for which even heuristic solutions are slow~\cite{jsma}.

To address the first issue, we propose two algorithms. 
The first is \emph{Greedy Marginal Benefit per Unit Cost (GMBC)} algorithm.
Define $\Delta A(j|\cS) = A(\cS) - A(\cS\cup \{j\})$ for any set of stabilized neurons $\cS$; this is the marginal decrease in accuracy from stabilizing a neuron $j$ in addition to those in~$\cS$.
Similarly, define $\Delta R(j|\cS) = R(\cS \cup \{j\}) - R(\cS)$, the marginal increase in robustness from stabilizing $j$.
We can greedily choose neurons to stabilize in decreasing order of $\frac{\Delta R(j|\cS)}{\Delta A(j|\cS)}$, until the accuracy ``budget" is saturated (that is, as long as accuracy stays above the bound $\beta$).
A second alternative algorithm we propose is \emph{Greedy Marginal Benefit (GMB)}, which stabilizes neurons solely in the order of $\Delta R(j|\cS)$.
If $A(S)$ is monotone decreasing in the number of neurons, we can show that \emph{GMB} requires only a logarithmic number of accuracy evaluations (see\ifarXiv Appendix~\ref{S:fastgmb}\else\cite{arXivVersion}\fi).
In practice, we can also run both in parallel and choose the better solution of the two; indeed, if $R(\cS)$ and $A(\cS)$ are both monotone and submodular, with $A(\cS)$ having some additional structure, the resulting algorithm exhibits a known approximation guarantee~\cite{Zhang16}.
However, we must be careful since in fact $A(\cS)$ is not necessarily monotone, and consequently $\Delta A(j|\cS)$ can be negative.
To address this, we maintain a positive lower bound $\bar{a}$ on this quantity, and if $\Delta A(j|\cS) < \bar{a}$ (including if it is negative), we simply set it to $\bar{a}$.

To address the second issue, we propose using an analytic proxy for $R(\cS)$, defining it as the sum of the increase in robustness from stabilizing the individual neurons in $\cS$ (see Section~\ref{section:stabilization}).

\section{Experiments}\label{section:experiments}

\begin{figure*}[t]
	\centering
	\begin{subfigure}{0.31\textwidth}
		\includegraphics[width=\textwidth]{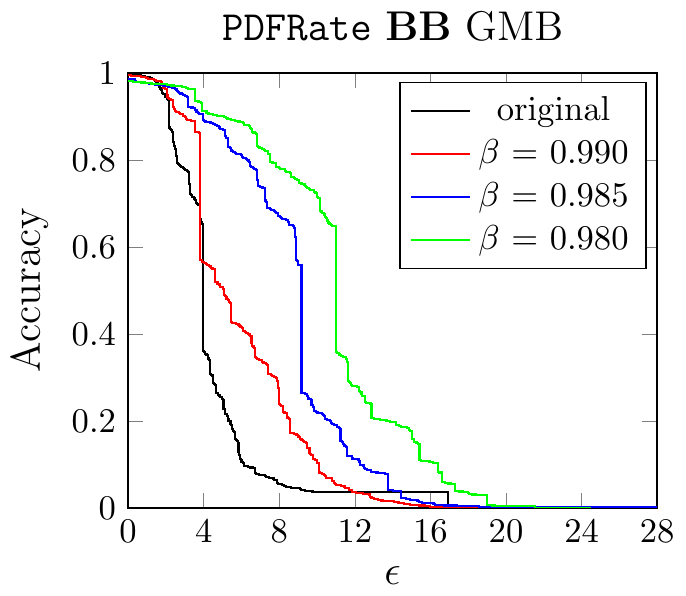}
	\end{subfigure}
	\begin{subfigure}{0.31\textwidth}
		\includegraphics[width=\textwidth]{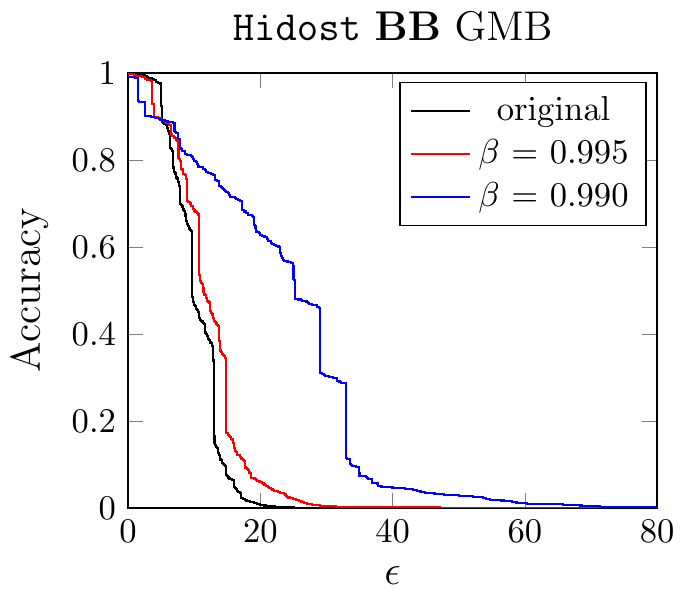}
	\end{subfigure}
	\begin{subfigure}{0.31\textwidth}
		\includegraphics[width=\textwidth]{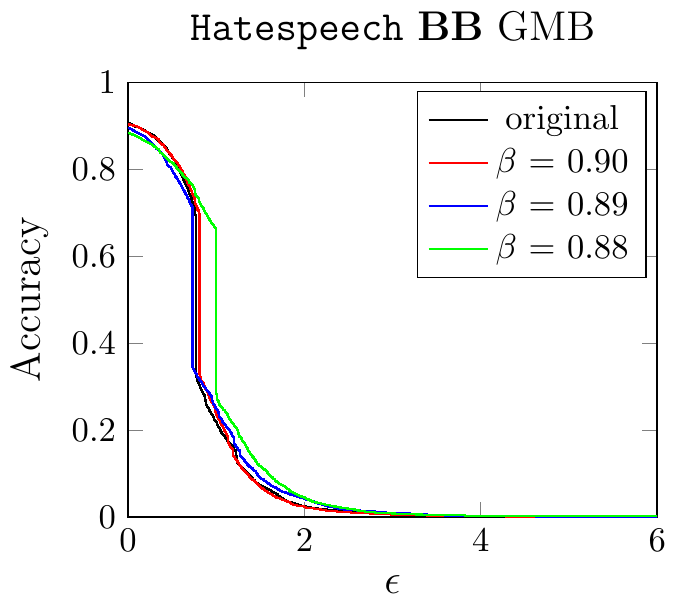}
	\end{subfigure}
	\begin{subfigure}{0.31\textwidth}
		\includegraphics[width=\textwidth]{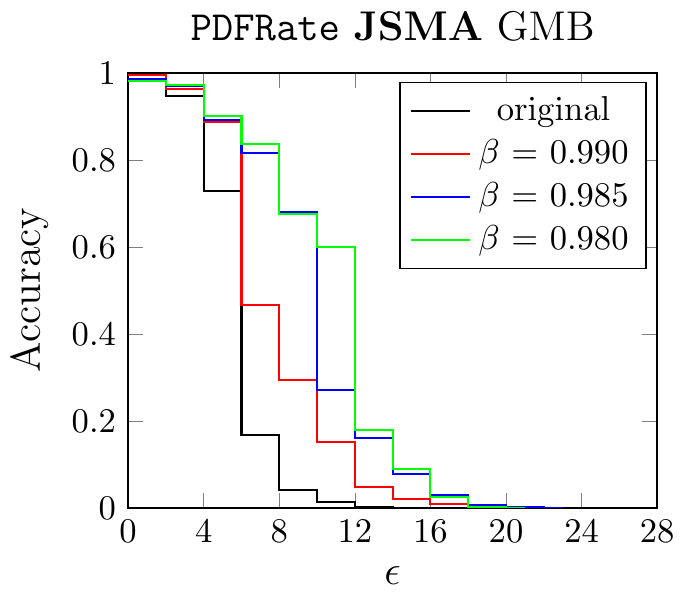}
	\end{subfigure}
	\begin{subfigure}{0.31\textwidth}
		\includegraphics[width=\textwidth]{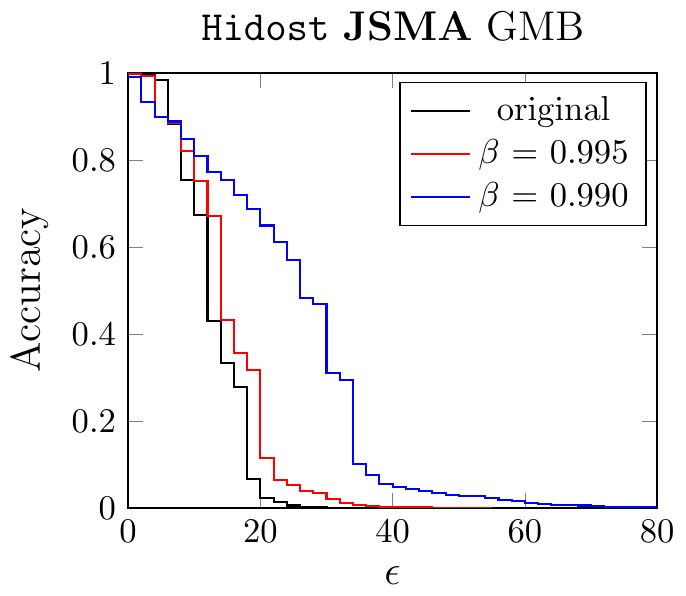}
	\end{subfigure}
	\begin{subfigure}{0.31\textwidth}
		\includegraphics[width=\textwidth]{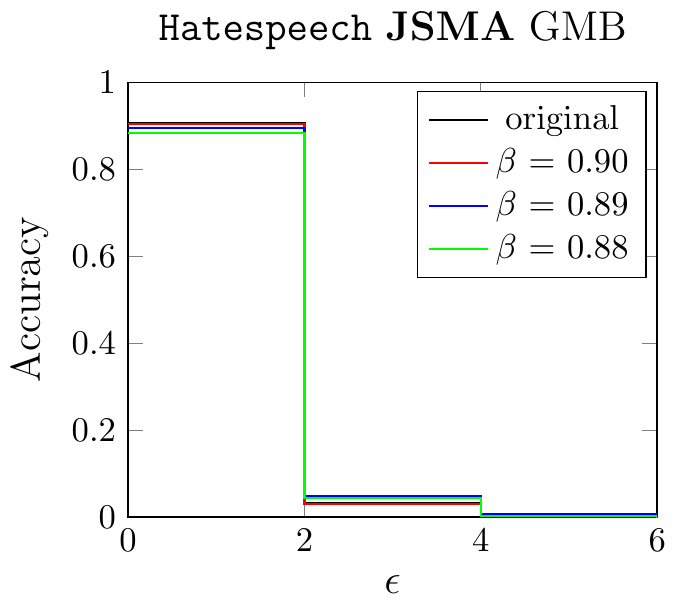}
	\end{subfigure}
	\caption{Robustness of original and stabilized neural network models (using \emph{GMB}) on PDFRate, Hidost, and Hate Speech datasets (columns) against the BB (top row) and JSMA (bottom row) attacks.  The $x$-axis shows varying levels of $\ell_1$ perturbation bound $\epsilon$ for the attacks.}
	\label{F:stabilize_base}
\end{figure*}



\paragraph{Datasets and Computing Infrastructure}
We evaluated the proposed approach using three security-related datasets: \emph{PDFRate}, \emph{Hidost}, and \emph{Hate Speech}.
The \emph{PDFRate} dataset~\cite{Tong35} is a PDF malware dataset which extracts features based on PDF file metadata and content.
The metadata features include the size of a file, author name, and creation date, while content-based features include position and counts of specific keywords.
This dataset includes 135 total features, which are then binarized if not already binary.
The \emph{Hidost} dataset~\cite{Tong38} is a PDF malware dataset which extracts features based on the logical structure of a PDF document.
Specifically, each binary feature corresponds to the presence of a particular \emph{structural path}, which is a sequence of edges in the reduced (tree) \emph{logical structure}, starting from the catalog dictionary and ending at this object (i.e., the shortest reference path to a PDF object).
This dataset is comprised of 658,763 PDF files and 961 features.

The \emph{Hate Speech} dataset~\cite{HateSpeech19}, collected from Gab, is comprised of conversation segments, with hate speech labels collected from Amazon Mechanical Turk workers.  This dataset contains 33,776 posts, and we used a bag-of-words binary representation with 200 most commonly occurring words (not including stop words).

All datasets were divided into training, validation, and test subsets; the former two were used for training and parameter tuning, while all the results below are using the test data.
We also used the validation set to select the subset of neurons $\cS$ to be stabilized.
For each dataset, we learned a two-layer sigmoidal fully connected neural network as a baseline.
Experiments were run on a research computer cluster with over 2,500 CPUs and 60 GPUs.

\paragraph{Attacks}
The robustness-accuracy tradeoff is quantified by the success rate of two state-of-the-art attacks, JSMA and~$\ell_1$-BB, under limited budget. \textit{Jacobian-based Saliency Map Attack} (JSMA)~\cite{jsma} (naturally adapted to the~$\{ \pm 1\}$ domain rather than~$\{0,1\}$), employs a greedy heuristic by which the bit with the highest impact is flipped. \textit{$\ell_1$ Brendel \& Bethge} ($\ell_1$-BB)~\cite{bb} is an attack that allows non-binary perturbations. It is radically different from JSMA in the sense that it requires an already-adversarial starting point which is then optimized. Given a clean point to attack, we select the adversarial starting point as the closest to it in~$\ell_1$-distance, among all points in the training set.


\paragraph{Adversarial Training}
In addition to the conventional baseline above, we also evaluated the use of neural network stabilization after adversarial training (\emph{AT})~\cite{Eugene}, which is still a state-of-the-art general-purpose approach for defense against adversarial example attacks.
We performed AT with the JSMA attack ($\ell_1$-norm $\epsilon=20$), which we adapted as follows: instead of minimizing the number of perturbed features to cause misclassification, we maximize loss subject to a constraint that we change at most $\epsilon$ features, still choosing which features to flip in the sorted order produced by JSMA.

\begin{figure*}[h]
	\centering
	\begin{subfigure}{0.31\textwidth}
		\includegraphics[width=\textwidth]{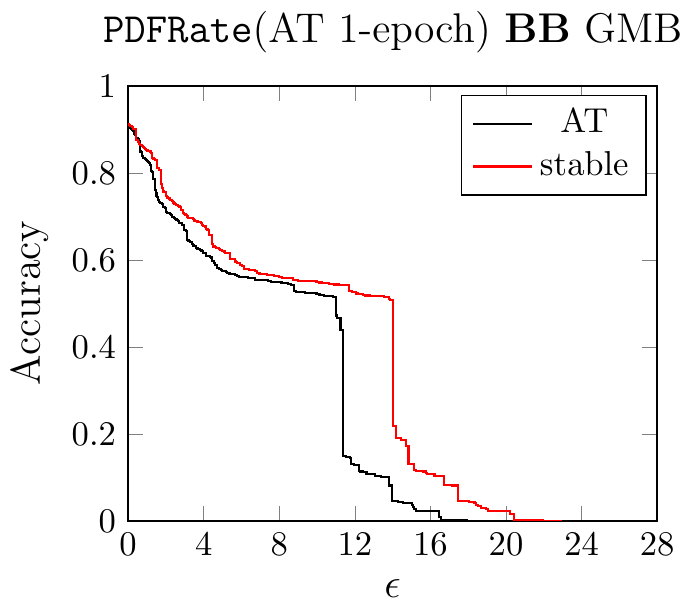}
	\end{subfigure}
	\begin{subfigure}{0.31\textwidth}
		\includegraphics[width=\textwidth]{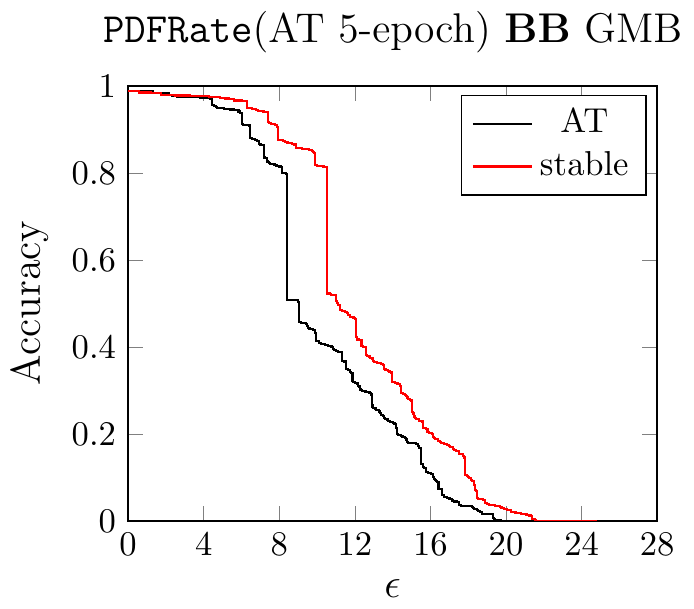}
	\end{subfigure}
	\begin{subfigure}{0.31\textwidth}
		\includegraphics[width=\textwidth]{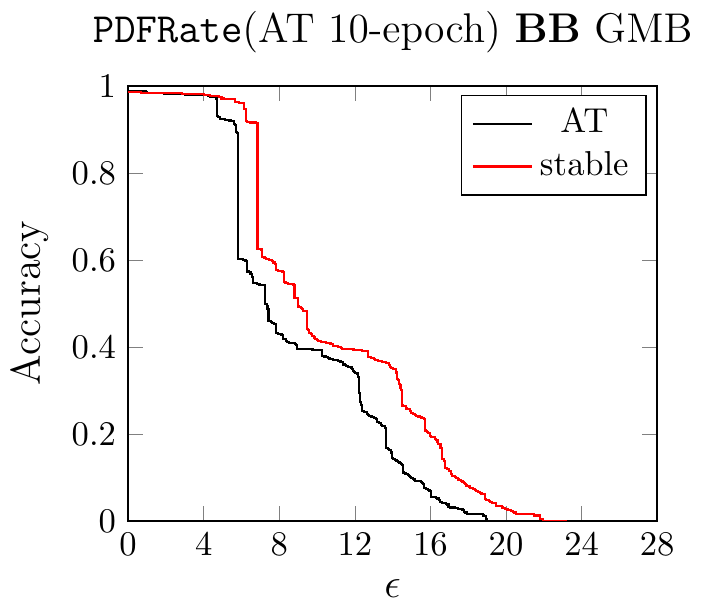}
	\end{subfigure}
	\begin{subfigure}{0.31\textwidth}
		\includegraphics[width=\textwidth]{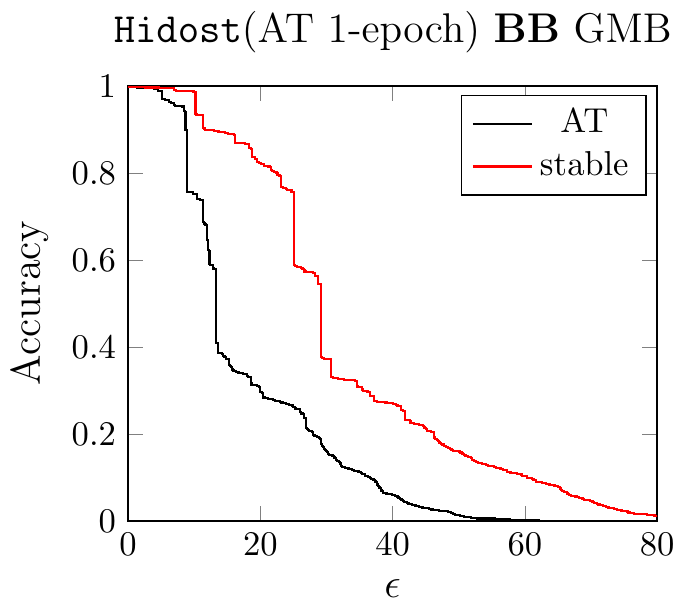}
	\end{subfigure}
	\begin{subfigure}{0.31\textwidth}
		\includegraphics[width=\textwidth]{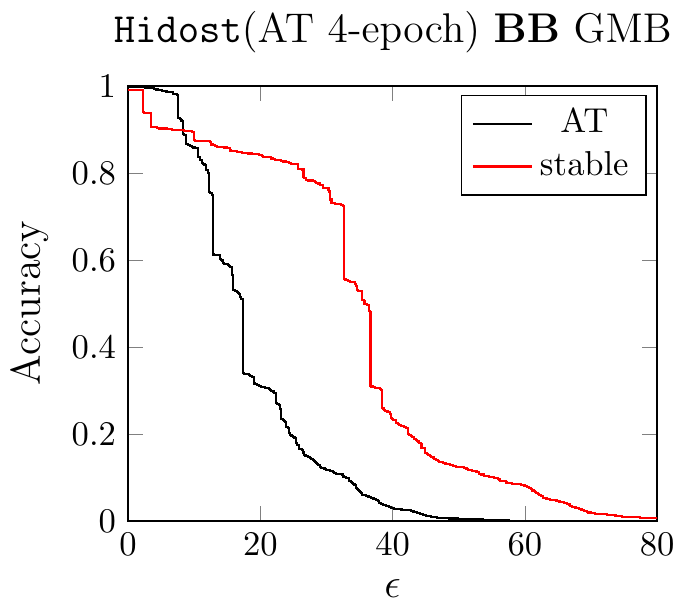}
	\end{subfigure}
	\caption{Robustness of adversarially trained neural networks and their stabilized variants (using \emph{GMB}). Top row: PDFRate dataset, after 1, 5, and 10 epochs of adversarial training (from left to right).  Bottom row: Hidost dataset after 1 (left) and 4 (right) epochs of adversarial training. The $x$-axis shows varying levels of $\ell_1$ perturbation bound $\epsilon$ for the attacks.}
	\label{F:stabilizeAT}
\end{figure*}

\subsection{Effectiveness of Neural Network Stabilization}

We first evaluate the proposed Fourier stabilization approach for neural network models on neural networks trained in a regular way on the \emph{PDFRate}, \emph{Hidost}, and \emph{Hate Speech} datasets.
The results are shown in Figure~\ref{F:stabilize_base} for the \emph{GMB} algorithm, where the top three plots (one for each dataset) are for the BB attack, and the bottom three are for the JSMA attack; results for \emph{GMBC} are provided in the supplement.
The most significant impact on robustness is in the case of the PDFRate dataset, where an essentially negligible drop in accuracy is accompanied by a substantial increase in robustness.
For example, for BB attack $\ell_1$ perturbation of at most $\epsilon = 10$ (the $x$-axis), robust accuracy ($y$-axis) increases from nearly 0 to 70\%, while clean data accuracy is 0.98.
We can observe a similar impact for the JSMA attack, with robust accuracy increasing from 0 to 60\%.
Fourier stabilization has a similarly substantial impact on the Hidost data: with accuracy still at 99\%, robust accuracy is increased from nearly 0 to 60\% for both the BB and JSMA attacks.
On the other hand, the impact is markedly small on the Hate Speech data, although even here we  see an increase in robust accuracy for BB attacks on the stabilized version for $\beta = 0.88$ and $\epsilon = 1$ from 30\% (baseline) to nearly 70\% (Fourier stabilization).

\subsection{Stabilizing Adversarially Trained Models}

In addition to demonstrating the value of stabilization for regularly trained neural networks (for example, when adversarial training is not an option, such as when datasets on which the original model was trained are sensitive), we now show that the approach also effectively composes with adversarial training (AT).
Figure~\ref{F:stabilizeAT} presents the results of stabilization (using \emph{GMB}; see the supplement for \emph{GMBC}) performed after several epochs of AT.
In all cases we see some improvement, and in a number of them the improvement over AT is considerable.
For example, on the Hidost dataset after 4 epochs of AT, robust accuracy is considerably improved by AT compared to the original model in Figure~\ref{F:stabilize_base}, but then further improved significantly by the proposed stabilization approach. For example, for $\epsilon=24$, robust accuracy increases from approximately 20\% to 80\%.

\section{Discussion}\label{section:discussion}
We introduced Fourier stabilization, a harmonic-analysis inspired post-training defense against adversarial perturbations of randomly chosen binary inputs. 
It is natural to consider extensions of this work in several fronts, e.g., worst-case robustness, non-uniform binary inputs, and real-valued inputs. In worst-case robustness, correct computation is required for \textit{every} input, i.e., $\bE_\boldx$ in~\eqref{equation:robustnessDef} is replaced by~$\min_{\boldx}$. While average-case robustness is more suited for applications such as malware detection, worst-case robustness is relevant in critical applications such as neuromorphic computing. It was recently shown in~\citet{CodNNarXiv} that worst-case robustness is impossible even against one bit erasure (i.e., setting~$x_i=0$ for some~$i$), unless redundancy is added, and a simple methods of adding such redundancy was given.

Extensions for non-uniform-binary or real-valued inputs require developing new tools in harmonic analysis. In the binary case, one needs to study the coefficients which come up instead of the Fourier ones, and if Plancherel's identity holds. In the real-valued case, e.g., when the inputs are distributed by a multivariate Gaussian, \textit{Hermite coefficients} can be used similarly, see \cite{AnalysisOfBoolean}, Sec.~11.2. However, in this case every neuron is already stabilized (see~\cite{TestingHalfspaces}, Prop.~25.2), and hence we suggest to consider other input distributions that are common in the literature, such as Gaussian mixture, and study the resulting coefficients.

\section*{Acknowledgements}
Y.~Vorobeychik was supported in part by NSF grants no. IIS-1905558 and ECCS-2020289, and ARO grant no. W911NF1910241.

\bibliography{ICML}
\bibliographystyle{icml2021}
\clearpage
\appendix
\input{supplement}

\end{document}

%% file: supplement.tex
\section*{Supplement to \emph{Enhancing Robustness of Neural Networks through Fourier Stabilization}}

\section{Omitted Proofs}\label{appendix:omittedProofs}


\begin{proof}[Proof of Lemma~\ref{lemma:LTFfacts}]
	We begin by introducing the notion of \textit{influences}~\cite{AnalysisOfBoolean} (Def.~2.13). The influence of coordinate~$i\in[n]$ is~$\infl_i[h]=\Pr[h(\boldx)\ne h(\boldx^{\oplus i})]$, where~$\boldx\in\{ \pm1 \}^n$ is chosen uniformly at random, and~$\boldx^{\oplus i}$ equals~$\boldx$ with its~$i$'th coordinate flipped. According to~\cite{AnalysisOfBoolean} (Ex.~2.5), we have that~$\infl_i[h]=|\hat{h}_i|$ for every~$i$ since~$h$ is unate\footnote{A Boolean function~$f:\{ \pm1  \}^n\to\{ \pm1 \}$ is called unate if it is monotone or anti-monotone in all~$n$ coordinates. The function~$f$ is monotone in coordinate~$i$ if $f(\boldx)\le f(\boldx^{\oplus i})$ for every~$\boldx$ such that~$x_i=-1$. Similarly, it is anti-monotone in coordinate~$i$ if $f(\boldx)\ge f(\boldx^{\oplus i})$ for every~$\boldx$ such that~$x_i=-1$. It is readily verified that every~$\sign$ function is unate.}. Therefore, for every~$i\in[n]$, we have that~$h$ depends on~$x_i$ if and only if~$\hat{h}_i\ne 0$. Now, observe that
	\begin{align*}
		\hat{h}_i=\;&\bE[x_ih(\boldx)]=\underbrace{\sum_{\boldx\vert x_i=1}\sign\left(\sum_{j\ne i}w_jx_j-(\theta-w_i)\right)}_{\triangleq A}\\
		&-\underbrace{\sum_{\boldx\vert x_i=-1}\sign\left(\sum_{j\ne i}w_jx_j-(\theta+w_i)\right)}_{\triangleq B},
	\end{align*}
	and hence, if~$w_i>0$, then $\theta+w_i>\theta-w_i$, and hence~$A\ge B$ and~$\hat{h}_i\ge 0$. Similarly, if~$w_i<0$, it follows that~$\hat{h}_i\le 0$. Since~$h$ depends on all its variables it follows that~$\hat{h}\ne 0$, and the claim follows.
\end{proof}

\begin{proof}[Proof of Lemma~\ref{lemma:stabWorks}]
	For simplicity, assume that~$\qnorm{\boldw}=\qnorm{\boldw^*}=1$; this can be assumed since scaling the weights (including the bias) does not change the accuracy nor the robustness. Also, let~$d_p^s$ be the signed variant of~$d_p$, i.e.,
	\begin{align*}
	    d_p^s(\boldx,\cH(\boldw,\theta))=\frac{\boldx\boldw^\intercal-\theta}{\qnorm{\boldw}}.
	\end{align*}
	According to Theorem~\ref{theorem:lp}, and by the definition of robustness~\eqref{equation:robustnessDef} and of signed distance, it follows that
	\begin{align*}
		\bE_{\boldx}&d_p(\boldx,\cH(\boldw,\theta))=\bE_{\boldx}d_p^s(\boldx,\cH(\boldw,\theta))h(\boldx)\\
		&\overset{(a)}{\le} \bE_{\boldx}d_p^s(\boldx,\cH(\boldw^*,\theta)) h(\boldx)\\
		&=\sum_{\boldx\vert h(\boldx)=h'(\boldx)}\Pr(\boldx)d_p^s(\boldx,\cH(\boldw^*,\theta)) h'(\boldx)-\\
		&\phantom{==}\sum_{\boldx\vert h(\boldx)\ne h'(\boldx)}\Pr(\boldx)d_p^s(\boldx,\cH(\boldw^*,\theta)) h'(\boldx)\\
		&=\sum_{\boldx\vert h(\boldx)=h'(\boldx)}\Pr(\boldx)d_p(\boldx,\cH(\boldw^*,\theta)) -\\
		&\phantom{==}\sum_{\boldx\vert h(\boldx)\ne h'(\boldx)}\Pr(\boldx)d_p(\boldx,\cH(\boldw^*,\theta)) \\
		&\overset{(b)}{\le}\bE_{\boldx}d_p(\boldx,\cH(\boldw^*,\theta)),
	\end{align*}
	where~$(a)$ follows from~$\boldw^*$ being the maximizer of the corresponding expression, and~$(b)$ follows from the positivity of distance. The ``in particular'' part follows from Theorem~\ref{theorem:main} since the expression after~$(a)$ is the objective function of the optimization problem, evaluated at its maximizer~$\boldw^*$, which results in~$\pnorm{\hat{\boldh}}-\hat{h}_\varnothing\theta$.
\end{proof}

\begin{proof} [Proof of Theorem~\ref{theorem:main}]
	The proof is split to the cases~$1<p<\infty$ and~$p=\infty$.
	
	\noindent\underline{The case~$1<p<\infty$:} Since the objective function and the constraint are differentiable, we use Lagrange multipliers. Define an additional variable~$\lambda$, and let
	\begin{align*}
		\ell(\boldv,\lambda)&=\hat{\boldh}\boldv^\intercal-\hat{h}_\varnothing\mu-\lambda(\norm{\boldv}_q^q-1)
	\end{align*}
	
	To find the extrema of~$\ell(\boldv,\lambda)$, we compute its gradient\footnote{Since~$1<p<\infty$, it follows that~$1<q<\infty$, and hence the function~$|x|^{q}$ is differentiable everywhere (including~$x=0$), and its derivative is~$qx|x|^{q-2}$.} with respect to derivation by~$(v_1,\ldots,v_n,\lambda)$,
	\begin{align*}
		\nabla_{\boldv,\lambda}&\ell(\boldv,\lambda)=(\hat{\boldh},0)-\\
		&\phantom{=}(\lambda q v_1|v_1|^{q-2},\ldots,\lambda q v_n|v_n|^{q-2},\norm{\boldv}_q^q-1)=0.
	\end{align*}
	and hence $\hat{h}_i=\lambda q v_i |v_i|^{q-2}=\lambda q\cdot \sign(v_i)\cdot|v_i|^{q-1}$ for every~$i\in[n]$. 
	Since the maximizer~$\boldw^*$ of~$\hat{\boldh}\boldv^\intercal$ clearly satisfies~$\sign(\hat{h}_i)=\sign(w_i^*)$ for every~$i\in[n]$, it follows that~$|\hat{h}_i|=\lambda q\cdot|w^*_i|^{q-1}$, i.e.,~$|w^*_i|=(|\hat{h}_i|/\lambda q)^{1/(q-1)}$. 
	By plugging this into~$\norm{\boldv}_q^q-1=0$,  if~$\lambda\ne 0$ then 
	\begin{align*}
		\sum_{i=1}^n \left(\frac{|\hat{h}_i|}{\lambda q}\right)^{\frac{q}{q-1}}&=1\\
		\lambda^{\frac{q}{q-1}}&=\sum_{i=1}^n \left(\frac{|\hat{h}_i|}{q}\right)^{\frac{q}{q-1}},
	\end{align*}
	and therefore
	\begin{align*}
		\lambda &= \left( \sum_{i=1}^n \left(\frac{|\hat{h}_i|}{q}\right)^{\frac{q}{q-1}} \right)^{\frac{q-1}{q}}=\left( \sum_{i=1}^n \left(\frac{|\hat{h}_i|}{q}\right)^{p} \right)^{\frac{1}{p}}\\
		&=\frac{1}{q}\left( \sum_{i=1}^n |\hat{h}_i|^p \right)^{\frac{1}{p}}=\frac{\norm{\hat{\boldh}}_p}{q}.
	\end{align*}
	Hence, the solution satisfies
	\begin{align}\label{equation:viabs}
		|w^*_i|&= \left( \frac{|\hat{h}_i|}{ \frac{\norm{\hat{\boldh}}_p}{q} \cdot q} \right)^{\frac{1}{q-1}}= \left( \frac{|\hat{h}_i|}{ \norm{\hat{\boldh} }_p }\right)^{ \frac{1}{q-1} }\nonumber\\
		&=\left( \frac{|\hat{h}_i|}{ \norm{\hat{\boldh} }_p }\right)^{p-1}.
	\end{align}
	Again, since~$\sign(w^*_i)=\sign(\hat{h}_i)$ for every~$i\in[n]$, it follows from~\eqref{equation:viabs} that $w^*_i=\sign(\hat{h}_i)(|\hat{h}_i|/\pnorm{\hat{\boldh}})^{p-1}$. If~$\lambda=0$ then~$\hat{\boldh}=0$, and then~$h$ must be constant\footnote{The famous Chow theorem~\cite{AnalysisOfBoolean} (Thm.~5.1) states that~$\sign$ functions (also known as \textit{Linear Threshold Functions}) are uniquely determined by their Chow parameters (see Section~\ref{section:FourierBackground}). Therefore, since the function~$c(\boldx)=1$ clearly has~$\hat{\boldc}=0$, it follows that~$h(\boldx)=c(\boldx)=1$.}. Finally, the resulting objective can be easily computed.
	
	\noindent\underline{The case~$p=\infty$:} For~$p=\infty$ the constraint~$\onenorm{\boldv}=1$ is not differentiable. However, notice that~$\onenorm{\boldv}\le 1$ defines a convex polytope whose vertices are~$\{ \pm \bolde_i \}_{i=1}^n$, where~$\bolde_i$ is the~$i$'th unit vector. Similar to the case~$p=1$, it is known that the optimum of a linear function over a convex polytope is obtained at a vertex. Therefore, it is readily verified that the solution is~$\boldw^*=\sign(\hat{h}_{i_{\text{max}}})\bolde_{i_\text{max}}$, where~$i_{\text{max}}\triangleq \argmax_{i\in[n]}|\hat{h}_i|$, for which the resulting objective is~$\hat{\boldh}\boldv^\intercal-\hat{h}_\varnothing\mu=\infnorm{\hat{\boldh}}-\hat{h}_\varnothing\mu$. 
\end{proof}
\section{Uniform and Binary Feature Extraction}\label{appendix:uniform}
As mentioned earlier, our Fourier analytic methods are applicable only in settings where the inputs presented to the adversary are binary, and uniformly distributed. While this is not a standard setting in adversarial machine learning, we point out cases in which this uniform binary distribution can be attained with little additional effort. We focus on settings where the extraction of features from real-world instances is freely chosen by the learner, such as in cybersecurity. Furthermore, it has been demonstrated in the past~\cite{tong2019improving} that binarization of features is beneficial to several applications in cybersecurity, which all the more correlates with our techniques. 

Consider a setting of defending against adversarial evasion attacks, in which the learner begins by extracting features from malicious and benign instances. Since the extraction of features from instances is up to the learner to decide, one can imagine every instance as a (potentially infinite) vector over the reals, out of which the learner focuses on a finitely many. Therefore, the instance space can be seen as~$\bR^n$ for some integer~$n$, where instances are sampled according to jointly Gaussian vector~$X$. 

To extract binary and uniform features from~$X$, we begin by calculating its covariance matrix~$\boldC=\bE[X^\intercal X]$; if not known a priori it can be approximated from the data. Then, finding the diagonalization~$\boldC=\boldU \boldD \boldU^\intercal$, where~$\boldD$ is diagonal and~$\boldU$ is unitary, allows us to decorrelate the features---it is an easy exercise to verify that the entries of~$X\boldU$ are uncorrelated. Finally, we binarize~$X\boldU$ by thresholding on the mean of its individual entries:
\begin{align*}
	bin(X\boldU)_j=\begin{cases}
		\phantom{-}1&\mbox{if }(X\boldU)_j\ge \bE[(X\boldU)_j]\\
		-1&\mbox{if }(X\boldU)_j< \bE[(X\boldU)_j]
	\end{cases}.
\end{align*}
It is readily verified that the distribution~$bin(X\boldU)$ is uniform over~$\{ \pm 1 \}^n$.
\section{Loss of Accuracy for~$1<p<\infty$}\label{appendix:Loss}
In this section we extend Theorem~\ref{theorem:onenorm} to other values of~$p$. All values~$1<p<\infty$ are covered by the discussion in this section. The case~$p=\infty$, which is of lesser interest due to drastic loss of accuracy, can be obtained by a variant of the proof of Theorem~\ref{theorem:onenorm}, and the details are left to the reader. To provide a bound similar to Theorem~\ref{theorem:onenorm} for~$1<p<\infty$, the following lemma is required.

\begin{lemma}\label{lemma:C1}
	Let~$\ell(\boldx)=\sum_{i=1}^n a_ix_i$, with~$\sum_{i=1}^n a_i^2=1$ and~$|a_i|\le\epsilon$. If the entries of~$\boldx$ are chosen uniformly at random, then there exist a $C_1\approx21.82$ such that for every~$\mu\ge0$,
	$$\bE[|\ell(\boldx)-\mu|]\le E_\mu+\rho\epsilon$$
	where~$\rho\triangleq\frac{4\pi C_1}{3\sqrt{3}}$, and~$E_\mu\triangleq \bE[|N(\mu,1)|]$ is the mean of a folded Gaussian.
\end{lemma}
To prove Lemma~\ref{lemma:C1}, the following version of the Central Limit Theorem is given.

\begin{theorem} (Berry-Esseen Theorem) \label{theorem:BEthm2} \cite{AnalysisOfBoolean} (Ex.~5.16, 5.31(d))
	Let~$X_1,\ldots,X_n$ be independent random variables with~$\bE[X_i]=0$, $|X_i|\le \epsilon$, and~$\var[X_i]=\sigma_i^2$ for every~$i\in[n]$, where~$\sum_{i=1}^n\sigma_i^2=1$. Then, for~$S=\sum_{i=1}^n X_i$, for every interval~$I\subseteq \bR$, and every~$u>0$, there exist absolute constants~$C_0,C_1$ such that
	\begin{align*}
		|\Pr[S\in I]-\Pr[N(0,1)\in I]|&\le 2C_0\epsilon,\mbox{ and}\\
		|\Pr[S\le u]-\Pr[(N(0,1)\le u]|&\le C_1\epsilon \cdot \frac{1}{1+|u|^3}.
	\end{align*}
\end{theorem}
Optimal values for~$C_0$ and~$C_1$ are not known, but current best estimates are~$C_0\approx0.47$ and~$C_1\approx 21.82$~\cite{BEref1,BEref2}.

\begin{proof} [Proof of Lemma~\ref{lemma:C1}]
	Following the proof of~\cite{TestingHalfspaces} (Prop.~32), with minor adjustments, we have
	\begin{align}\label{equation:someIntegrals}
		\bE&[|\ell(\boldx)-\mu|]=\int_{0}^{\infty}\Pr[|\ell(\boldx)-\mu|>s]ds\nonumber\\
		&=\int_{0}^{\infty}\Pr[\ell(\boldx)>\mu+s]+\Pr[\ell(\boldx)<\mu-s]ds\nonumber\\
		&=\int_{0}^{\infty}\Pr[N(0,1)>\mu+s]+\Pr[N(0,1)<\mu-s]ds\nonumber\\&\phantom{=}+C_1\epsilon\int_{0}^{\infty}\frac{1}{1+|\mu+s|^3}+\frac{1}{1+|\mu-s|^3} ds\nonumber\\
		&=\int_{0}^{\infty}\Pr[|N(0,1)-\mu|>s]ds\nonumber\\
		&\phantom{=}+C_1\epsilon\int_{0}^{\infty}\frac{1}{1+|\mu+s|^3}ds+C_1\epsilon\int_{0}^{\infty}\frac{1}{1+|\mu-s|^3} ds.
	\end{align}
	The leftmost integral in~\eqref{equation:someIntegrals} equals~$E_\mu$ by definition, and a variable substitutions of~$x=\mu+s$ and~$x=\mu-s$ in the remaining two, respectively, yields
	\begin{align*}
		\eqref{equation:someIntegrals}=E_\mu+C_1\epsilon\int_{-\infty}^{\infty}\frac{1}{1+|x|^3}dx=E_{\mu}+\frac{4\pi C_1\epsilon}{3\sqrt{3}},
	\end{align*}
	where the last equality is a known formula. 
\end{proof}

We now turn to bound the accuracy for~$\ell_p$-Fourier stabilization with~$1<p<\infty$.

\begin{theorem} \label{theorem:lpDeviation}
	For~$h(\boldx)=\sign(\boldx\boldw^\intercal-\theta)$, let~$\ell(\boldx)=\frac{1}{\sigma}\boldx\boldw^{*\intercal}$, where~$w^*_i=\sign(\hat{h}_i)\left( \frac{|\hat{h}_i|}{\norm{\hat{\boldh}}_p} \right)^{\frac{1}{q-1}}$ and~$\sigma=\twonorm{\boldw^*}$, and for any~$\mu>0$ let
	\begin{align}\label{equation:lpDeviationEquation}
		\gamma=\gamma(\mu)=\left|\left(\frac{\norm{\hat{\boldh}}_p^p}{\twonorm{\hat{\boldh}^{\frac{1}{q-1}}}}-\hat{h}_\varnothing\mu\right)-E_\mu\right|.
	\end{align}
	Then,
	\begin{align}\label{equation:AccLossBound}
		\Pr(\sign(\ell(\boldx)&-\mu)\ne h(\boldx))\le\nonumber\\\
		&\tfrac{3}{2}\left(C_0\epsilon+\sqrt{C_0^2\epsilon^2+\sqrt{\tfrac{2}{\pi}}\cdot(\gamma+\rho\epsilon)}\right),
	\end{align}
	where~$\rho=\frac{4\pi C_1}{3\sqrt{3}}$  and~$\epsilon=\frac{1}{\sigma}\max\{ |w^*_i| \}_{i=1}^n$. 
\end{theorem}


\begin{proof} 
	First, notice that 
	\begin{align}\label{equation:sigma}
		\sigma=\sqrt{\sum_{i=1}^n \left( \frac{|\hat{h}_i|}{\pnorm{\hat{\boldh}}} \right)^{\frac{2}{q-1}} }=\pnorm{\hat{\boldh}}^{\frac{1}{1-q}}\cdot\twonorm{\hat{\boldh}^{\frac{1}{q-1}}}.
	\end{align}
	Second, according to Plancherel's identity, 
	\begin{align}\label{equation:taupNorm}
		\bE&[h(\boldx)(\ell(\boldx)-\mu)]
		=\tfrac{1}{\sigma}\sum_{i=1}^n \hat{h}_i \sign(\hat{h}_i)\left( \frac{|\hat{h}_i|}{\pnorm{\hat{\boldh}}} \right)^{\frac{1}{q-1}}-\hat{h}_\varnothing\mu\nonumber\\
		&= \frac{1}{\sigma\norm{\hat{\boldh}}_p^{\frac{1}{q-1}}}\sum_{i=1}^n|\hat{h}_i|^p-\hat{h}_\varnothing\mu\nonumber\\
		&\overset{\eqref{equation:sigma}}{=}\frac{\pnorm{\hat{\boldh}}^p}{\pnorm{\hat{\boldh}}^{\frac{1}{1-q}}\cdot\twonorm{\hat{\boldh}^{\frac{1}{q-1}}}\cdot\norm{\hat{\boldh}}_p^{\frac{1}{q-1}}}-\hat{h}_\varnothing\mu\nonumber\\
		&=\frac{\pnorm{\hat{\boldh}}^p}{\twonorm{\hat{\boldh}^{\frac{1}{q-1}}}}-\hat{h}_\varnothing\mu.
	\end{align}
	Third, we have that
	\begin{align}\label{eqaution:alphan}
		\bE[h(\boldx)(\ell(\boldx)-\mu)]\overset{(a)}{\le} \bE[|\ell(\boldx)-\mu|]\overset{(b)}{\le}E_\mu+\rho\epsilon.
	\end{align}
	where~$(a)$ is since~$h(\boldx)\le 1$, and~$(b)$ is by Lemma~\ref{lemma:C1}. Therefore, by the definition of~$\gamma$, it follows that
	\begin{align}\label{equation:gammaUpperBound}
		\bE&[(\ell(\boldx)-\mu)\cdot(\sign(\ell(\boldx)-\mu)-h(\boldx))]=\nonumber\\
		&=\bE[|\ell(\boldx)-\mu|]-\bE[h(\boldx)(\ell(\boldx)-\mu)]\nonumber\\
		&\overset{(c)}{\le}E_\mu-\frac{\pnorm{\hat{\boldh}}^p}{\twonorm{\hat{\boldh}^{\frac{1}{q-1}}}}+\hat{h}_\varnothing\mu+\rho\epsilon\overset{(d)}{\le} \gamma+\rho\epsilon,
	\end{align}
	where~$(c)$ follows from~\eqref{equation:taupNorm} and~\eqref{eqaution:alphan}, and~$(d)$ from the definition of~$\gamma$~\eqref{equation:lpDeviationEquation}. In what follows, we bound~$\Pr(\sign(\ell(\boldx)-\mu)\ne h(\boldx))$ by studying the expectation in~\eqref{equation:gammaUpperBound}. To this end, notice that for every~$u>0$ (a precise~$u$ will be given shortly), Lemma~\ref{lemma:C0} implies that
	\begin{align}\label{equation:probell}
		\Pr(|\ell(\boldx)-\mu|\le u)\le 
		u\sqrt{\tfrac{2}{\pi}}+ 2C_0\epsilon\triangleq\eta(u).
	\end{align}
	Assume for contradiction that $\Pr(\sign(\ell(\boldx))\ne h(\boldx))> \frac{3}{2}\eta(u)$. Since~$\Pr(|\ell(\boldx)-\mu|>u)\ge 1-\eta(u)$ by~\eqref{equation:probell}, it follows that
	\begin{align}\label{equation:ProbBound}
		\Pr(\sign(\ell(\boldx)-\mu)\ne h(\boldx)\mbox{ and } |\ell(\boldx)-\mu|>u) > \tfrac{\eta(u)}{2}.
	\end{align}
	Now observe that
	\begin{align}\label{equation:Ebound}
		&\bE[(\ell(\boldx)-\mu)(\sign(\ell(\boldx)-\mu)-h(\boldx))]=\nonumber\\
		\nonumber\\&\tfrac{1}{2^n}\left(\sum_{\boldx\vert \sign(\ell(\boldx)-\mu)>h(\boldx)} 2(\ell(\boldx)-\mu) -\right. \nonumber\\
		&\left. \sum_{\boldx\vert \sign(\ell(\boldx)-\mu)<h(\boldx)}2(\ell(\boldx)-\mu)\right).
	\end{align}
	Since all summands in the left summation in~\eqref{equation:Ebound} are positive, and all summands in the right one are negative, by keeping in the left summation only summands for which~$\ell(\boldx)-\mu> u$, and in the right summation only those for which~$\ell(\boldx)-\mu<-u$, we get
	\begin{align}\label{equation:uetau}
		\eqref{equation:Ebound}&\ge 2u\cdot\frac{|\{ \boldx\vert \sign(\ell(\boldx)-\mu)\ne h(\boldx)\mbox{ and }|\ell(\boldx)-\mu|>u \}|}{2^n}\nonumber\\
		&\overset{\eqref{equation:ProbBound}}{>}u\cdot\eta(u).
	\end{align}
	
\begin{figure*}[h!]
	\centering
	\begin{subfigure}{0.31\textwidth}
		\includegraphics[width=\textwidth]{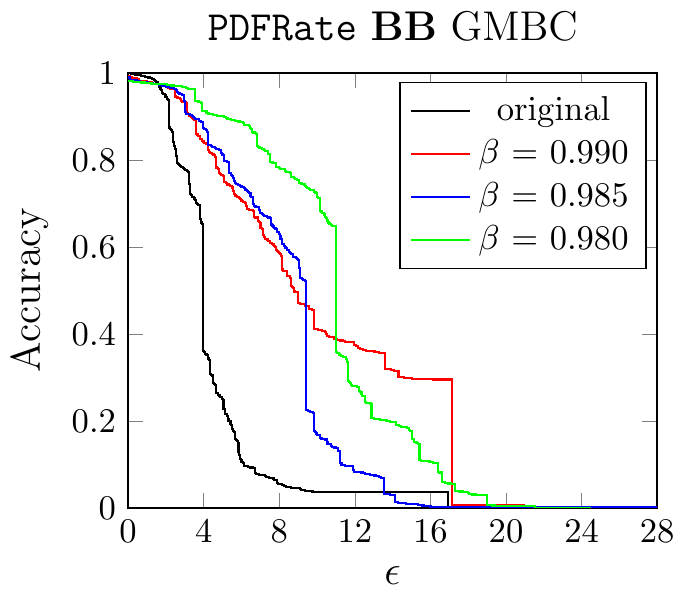}
	\end{subfigure}
	\begin{subfigure}{0.31\textwidth}
		\includegraphics[width=\textwidth]{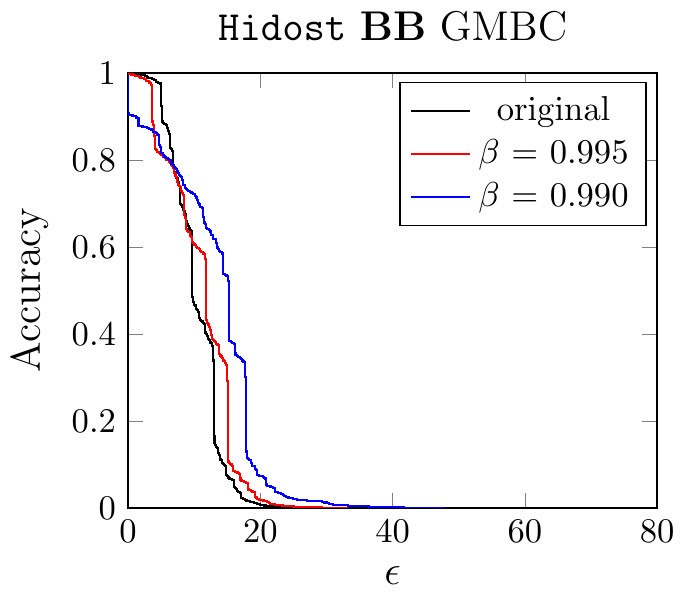}
	\end{subfigure}
	\begin{subfigure}{0.31\textwidth}
		\includegraphics[width=\textwidth]{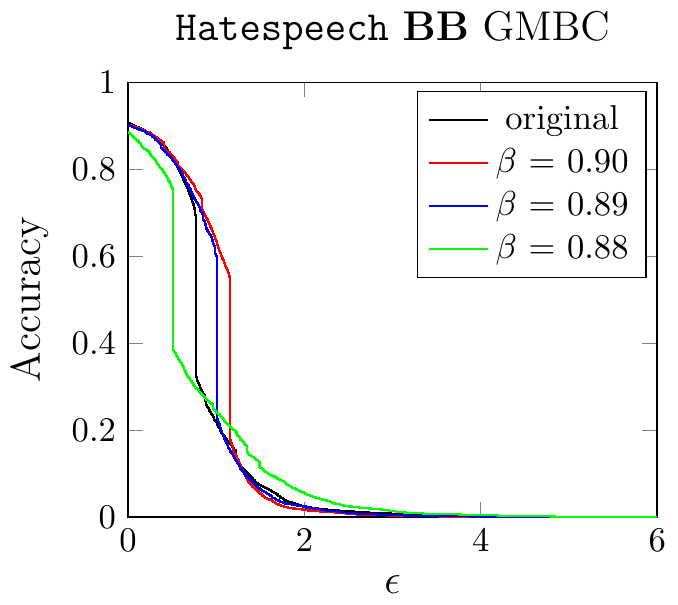}
	\end{subfigure}
	\begin{subfigure}{0.31\textwidth}
		\includegraphics[width=\textwidth]{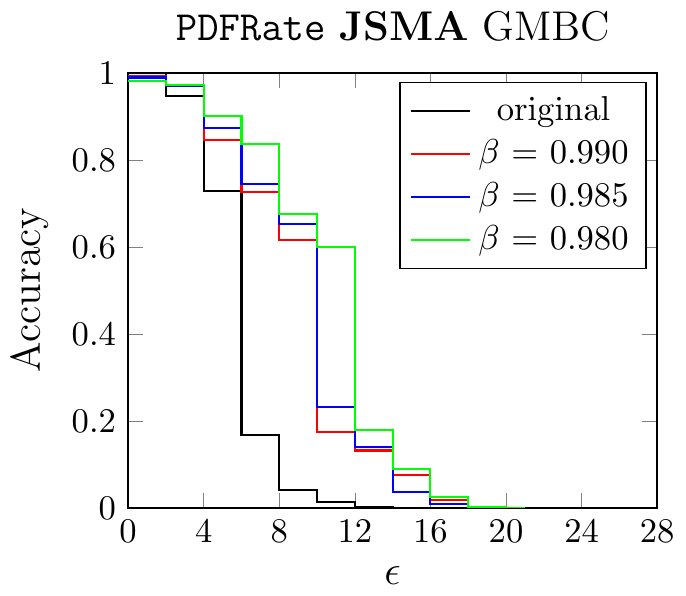}
	\end{subfigure}
	\begin{subfigure}{0.31\textwidth}
		\includegraphics[width=\textwidth]{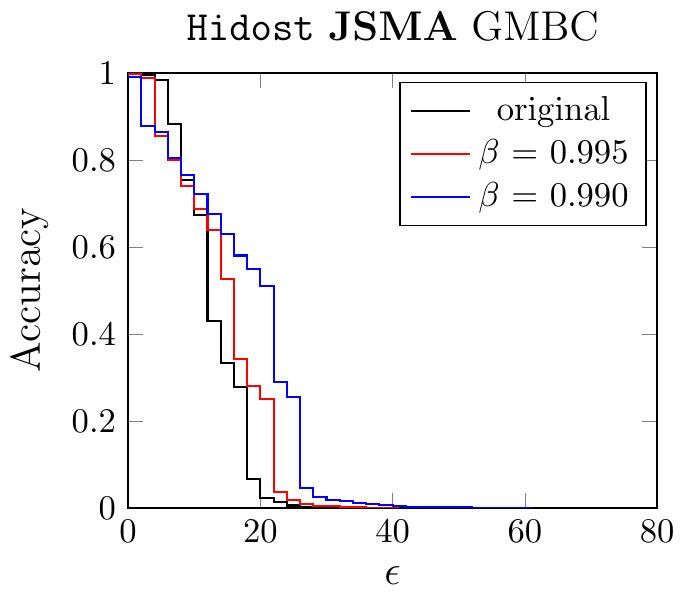}
	\end{subfigure}
	\begin{subfigure}{0.31\textwidth}
		\includegraphics[width=\textwidth]{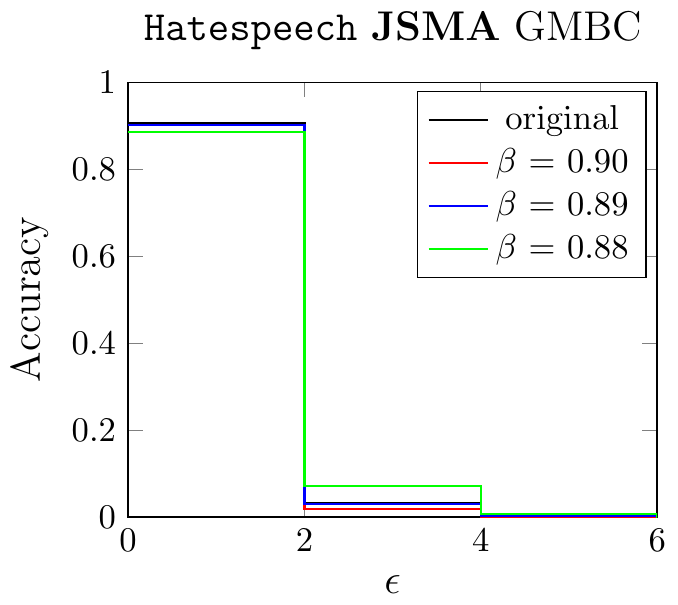}
	\end{subfigure}
	\caption{Robustness of original and stabilized neural network models with \textbf{sigmoid} (using \emph{GMBC}) on PDFRate, Hidost, and Hate Speech datasets (columns) against the BB (top row) and JSMA (bottom row) attacks.  The $x$-axis shows varying levels of $\ell_1$ perturbation bound $\epsilon$ for the attacks.}
		\label{F:gmbc}
\end{figure*}
	
	Combining~\eqref{equation:uetau} with~\eqref{equation:gammaUpperBound}, it follows that
	\begin{align*}
		u\cdot\eta(u)&< \gamma+\rho\epsilon
	\end{align*}
	which by the definition in~\eqref{equation:probell} implies that
	\begin{align}\label{equation:possibleCont}
		\sqrt{\tfrac{2}{\pi}}\cdot u^2+2C_0\epsilon\cdot u-(\gamma+\rho\epsilon)&< 0.
	\end{align}
	We wish to find the smallest positive value of~$u$ which contradicts~\eqref{equation:possibleCont}. Clearly, any positive~$u$ which complies with~\eqref{equation:possibleCont} must satisfy
	\begin{align}\label{equation:usValue}
		u&<\frac{-2C_0\epsilon+\sqrt{4C_0^2\epsilon^2+4\sqrt{\frac{2}{\pi}}\cdot(\gamma+\rho\epsilon)}}{2\sqrt{\frac{2}{\pi}}}
		\nonumber\\
		&=\frac{-C_0\epsilon+\sqrt{C_0^2\epsilon^2+\sqrt{\frac{2}{\pi}}\cdot(\gamma+\rho\epsilon)}}{\sqrt{\frac{2}{\pi}}},
	\end{align}
	and hence setting~$u$ to the rightmost expression in~\eqref{equation:usValue} leads to a contradiction. This implies that
	\begin{align*}
		\Pr&(\sign(\ell(\boldx)-\mu)\ne h(\boldx))\le \tfrac{3}{2}\eta(u)\overset{\eqref{equation:probell}}{=}\tfrac{3}{2}(u\sqrt{\tfrac{2}{\pi}}+ 2C_0\epsilon)\\
		&=\tfrac{3}{2}\left(-C_0\epsilon+\sqrt{C_0^2\epsilon^2+\sqrt{\tfrac{2}{\pi}}\cdot(\gamma+\rho\epsilon)}+2C_0\epsilon\right)\\
		&=\tfrac{3}{2}\left(C_0\epsilon+\sqrt{C_0^2\epsilon^2+\sqrt{\tfrac{2}{\pi}}\cdot(\gamma+\rho\epsilon)}\right).\qedhere
	\end{align*}
\end{proof}

\section{Additional Experiments}

\subsection{GMBC Algorithm}

In~Section~\ref{section:experiments}, we presented the results of neural network stabilization using the GMB algorithm which only uses accuracy in assessing when the accuracy constraint has been violated.
Here we present analogous results for using GMBC.
As we can see from Figure~\ref{F:gmbc}, overall the GMB algorithm is considerably more effective.
Indeed, if we use the blended algorithm in which we always run both GMB and GMBC and take the best solution of the two in terms of robustness, the result is equivalent to running GMB in our setting.

\begin{figure*}[t]
	\centering
	\begin{subfigure}{0.31\textwidth}
		\includegraphics[width=\textwidth]{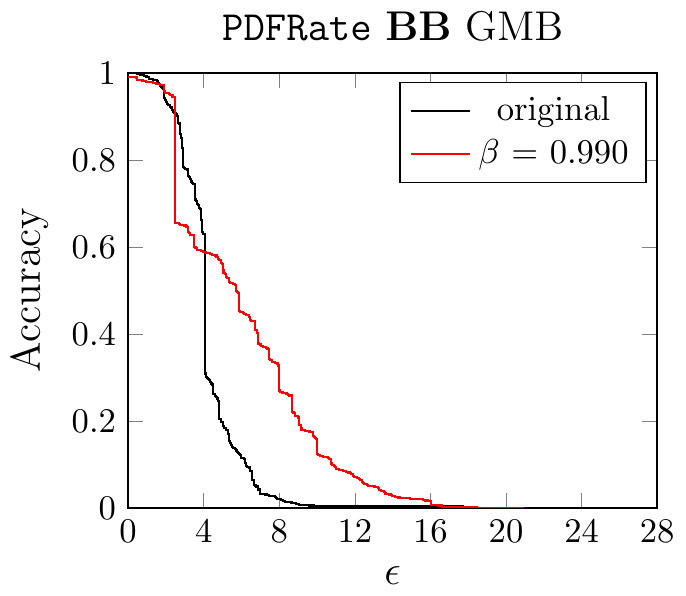}
	\end{subfigure}
	\begin{subfigure}{0.31\textwidth}
		\includegraphics[width=\textwidth]{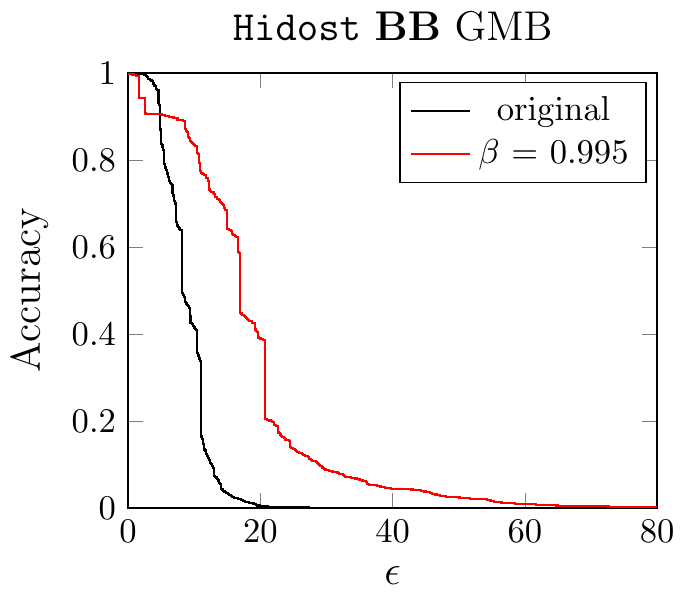}
	\end{subfigure}
	\begin{subfigure}{0.31\textwidth}
		\includegraphics[width=\textwidth]{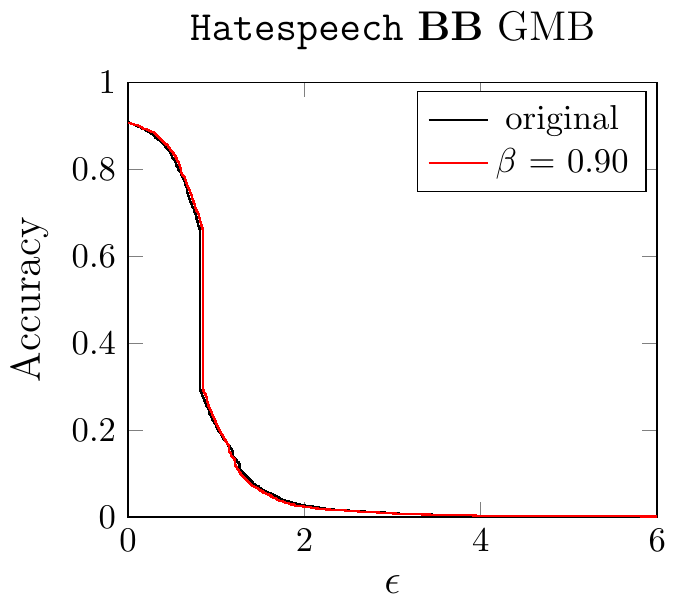}
	\end{subfigure}
	\begin{subfigure}{0.31\textwidth}
		\includegraphics[width=\textwidth]{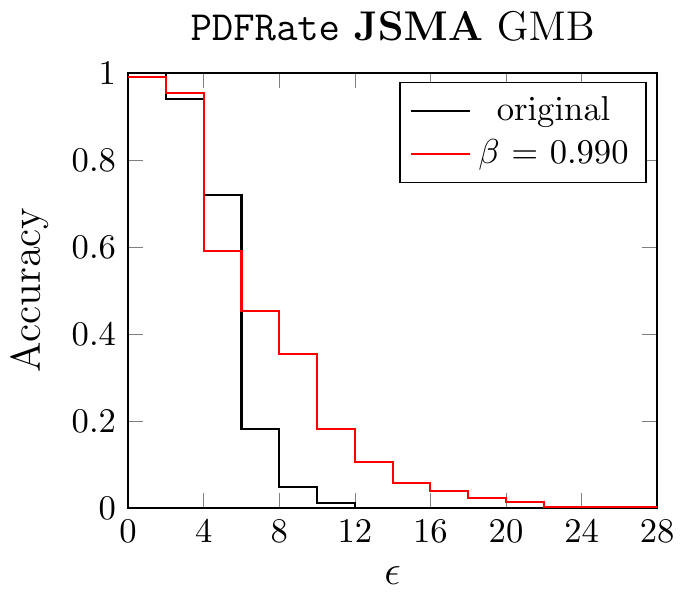}
	\end{subfigure}
	\begin{subfigure}{0.31\textwidth}
		\includegraphics[width=\textwidth]{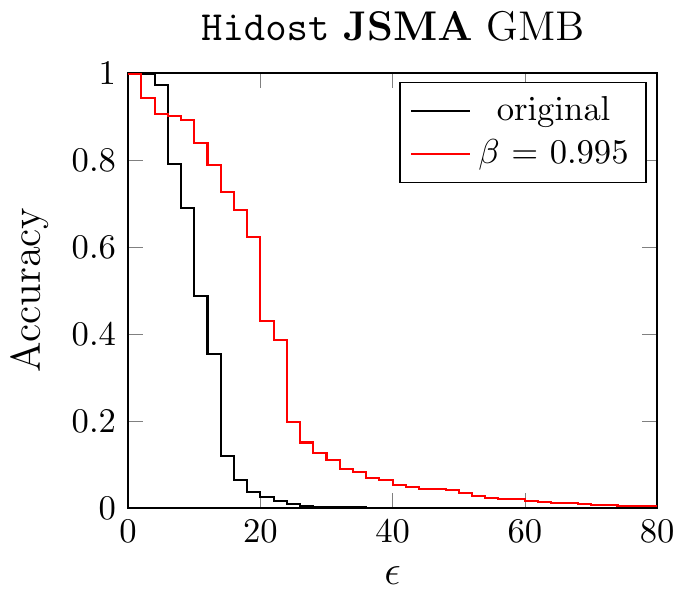}
	\end{subfigure}
	\begin{subfigure}{0.31\textwidth}
		\includegraphics[width=\textwidth]{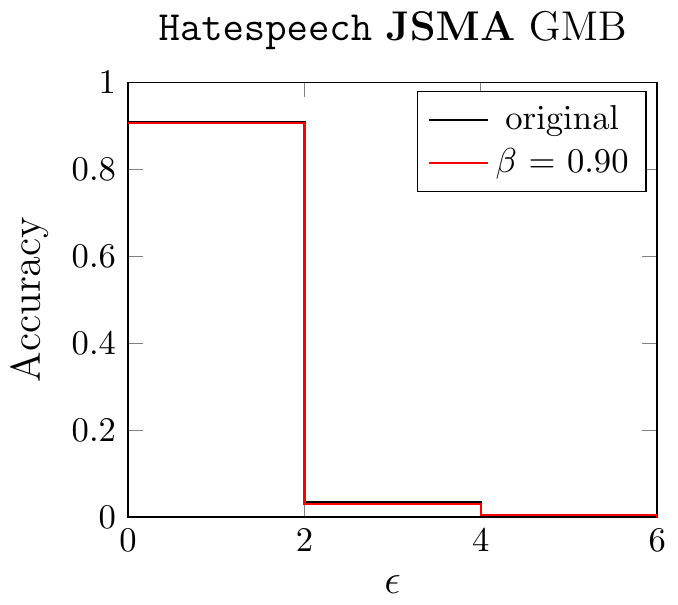}
	\end{subfigure}
	\caption{Robustness of original and stabilized neural network models with \textbf{ReLU} activations (using \emph{GMBC}) on PDFRate, Hidost, and Hate Speech datasets (columns) against the BB (top row) and JSMA (bottom row) attacks.  The $x$-axis shows varying levels of $\ell_1$ perturbation bound $\epsilon$ for the attacks.}
	\label{F:relu}
\end{figure*}

\subsection{ReLU Activation Function}

Our experiments in Section~\ref{section:experiments} used the \textbf{sigmoid} activation functions as neurons.
Here, we present results for neural networks that instead use the more prevalent \textbf{ReLU} activation functions.
As we can see from Figure~\ref{F:relu}, the results are qualitatively the same: stabilization considerably improves robustness of the networks.
However, the impact is somewhat smaller than for the sigmoidal neural networks, and stabilization appears to have no effect on the Hate Speech dataset in this case.

\section{Speeding up GMB}
\label{S:fastgmb}

If we assume that accuracy decreases monotonically as more neurons are stabilized, then \emph{GMB} can be rephrased as a search problem, which can be solvable via binary search. The key insight is that our proxy for computing change in robustness is based only on the weights of an individual neuron. Therefore, the order in which neurons are stabilized is computed before the algorithm begins. In \emph{GMB}, computing the accuracy of the model is the time-consuming step, and here we reduce the number of accuracy evaluations from $O(k)$ to $O(\log k)$, where $k$ is the size of the first layer of the network (number of neurons). Runtime experiment results can be found in Table~\ref{table:runtime}.

\begin{table*}[h]
    \centering
    \begin{tabular}{rllllll}
\hline
   $\beta$ & $16$ neurons   & $64$ neurons   & $256$ neurons   & $1024$ neurons   & $4096$ neurons   & $16384$ neurons   \\
\hline
      0.99 & $0.55$         & $0.60$         & $0.70$          & $1.37$           & $3.19$           & $9.89$            \\
      0.98 & $0.40$         & $0.56$         & $0.65$          & $1.30$           & $3.19$           & $10.51$           \\
      0.97 & $0.42$         & $0.56$         & $0.64$          & $1.32$           & $3.55$           & $9.57$            \\
\hline
\end{tabular}

    \caption{The running time of the algorithm outlined in Appendix~\ref{S:fastgmb} on a 2018 MacBook Pro. The algorithm was tested on networks classifying the \texttt{PDFRate} dataset with varying numbers of neurons on their first layer. For completeness, we also varied the accuracy threshold $\beta$, but we observe this made no significant impact on the run time.}
    \label{table:runtime}
\end{table*}


In \textrm{GMB}, we aim to maximize our proxy for robustness while keeping the accuracy above a threshold. At the beginning of the algorithm, we compute $\Delta R$ for each neuron, the increase in robustness caused by stabilizing that neuron, and aim to maximize the sum of the $\Delta R$s. We do this greedily by repeatedly stabilizing the next neuron with the largest $\Delta R$. Then, we order neurons from $h_1, \ldots, h_t$ based on decreasing $\Delta R$, and GMB stabilizes $h_1$, then $h_2$, and so forth, until we stabilize the largest $h_i$ such that the accuracy is still above the $\beta$ threshold.

It is evident that this problem is equivalent to the search problem of finding the largest~$i$ such that the accuracy is $\ge \beta$. By our monotonicity assumption, accuracy decreases with increasing~$i$, hence binary search is applicable. At each step of this binary search, we evaluate a given index~$i$. We stabilize all neurons $h_1, \ldots, h_i$ and then evaluate the accuracy of the model. 
If it is below $\beta$, we wish to stabilize fewer neurons, and if it is above $\beta$, we wish to stabilize more. 


We implemented GMBC with binary search and tested its runtime for networks classifying \texttt{PDFRate} with varying numbers of neurons in their hidden layer. All tests were run on a 2018 MacBook Pro. The results can be found in Table~\ref{table:runtime}.  As expected, we observe that it had insignificant effects on the run time. We additionally note that the trend does not appear logarithmic. This is due to the fact that accuracy evaluations take more time for large networks, in spite of conducting $O(\log k)$ accuracy evaluations.